\documentclass{article}
\usepackage{arxiv}

\usepackage[utf8]{inputenc}
\usepackage{amsthm,amsmath,amsfonts,amssymb,mathtools,xparse,bbm,graphicx,float,authblk}
\usepackage{hyperref}
\hypersetup
{
	colorlinks,
	linkcolor = blue,
	citecolor = blue,
	urlcolor = blue,
	filecolor = blue,
	breaklinks,
}

\newtheorem{theorem}{Theorem}
\numberwithin{theorem}{section}
\newtheorem{proposition}[theorem]{Proposition}
\newtheorem{lemma}[theorem]{Lemma}

\theoremstyle{definition}

\theoremstyle{remark}
\newtheorem{remark}[theorem]{Remark}
\newtheorem{example}[theorem]{Example}

\NewDocumentCommand{\gnp}{O{n} O{p}}{{\text {\bf G}}(#1, #2)}
\NewDocumentCommand{\xnp}{O{n} O{p}}{{\text {\bf X}}(#1, #2)}
\NewDocumentCommand{\multixnp}{O{n} O{p}}{{\text {\bf X}}(#1, \mathbf{#2})}
\NewDocumentCommand{\multixnpnotbold}{O{n} O{p}}{{\text {\bf X}}(#1, #2)}

\usepackage[usenames]{xcolor}

\newcommand{\gr}[1]{\textcolor{black}{{{}}#1}{}}

\usepackage{color}
\usepackage{comment}
\usepackage{enumitem}
\usepackage{booktabs}

\usepackage{dirtytalk}
\usepackage{amsmath}
\usepackage{amssymb}
\usepackage{multirow}
\usepackage{natbib}

\newcommand{\R}{{{\mathbb R}}} 
\newcommand{\E}{{{\mathbb E}}}
\renewcommand{\P}{{{\mathbb P}}}

\newcommand{\II}{{{\mathbb I}}}

\newcommand{\ns}{N(\mathcal{\tilde{S}})}

\newcommand{\qtilde}{\tilde{q}}

\newcommand{\Ftilde}{\tilde{F}}

\title{Split Conformal Prediction under Data Contamination}
\author[1]{%%%% Author details
Jase Clarkson}
\author[2]{Wenkai Xu}
\author[3,1]{Mihai Cucuringu} \author[4]{ Yvik Swan} \author[1] {Gesine Reinert}

\affil[1]{Department of Statistics, University of Oxford, OX1 3LB, United Kingdom}
\affil[2]{
Department of Statistics, University of Warwick, CV4 7AL, United Kingdom}\affil[3]{Department of Mathematics, UCLA, 
520 Portola Plaza, Los Angeles, CA 90095, USA}
\affil[4]
{Department of Mathematics, Université libre de Bruxelles, Boulevard du Triomphe B-1050,
Belgium
}

\begin{document}
\maketitle

%%%% Subject entries to be placed here %%%%

%%%% Insert corresponding author and its email address}

%%%% Abstract text to be placed here %%%%%%%%%%%%
\begin{abstract}
% \noindent 

Conformal prediction is a nonparametric technique for constructing prediction intervals or sets from arbitrary predictive models.  
It is popular because it comes with theoretical guarantees on the marginal coverage of the prediction sets when the data are 
exchangeable, and the split conformal prediction variant has low computational cost compared to model training.  We study the robustness of split conformal prediction when the data are contaminated: a fraction of the calibration scores are outliers, drawn from a  different distribution than the bulk. We quantify the impact of the corrupted data on the coverage and efficiency of the constructed sets when evaluated on ``clean'' test points, and confirm our findings  with numerical experiments. 
In the classification setting,  we propose an adjustment  which we call Contamination Robust Conformal Prediction, and verify the efficacy of our approach using both synthetic and real datasets.

\end{abstract}
%%%%%%%%%%%%%%%%%%%%%%%%%%%

%%%%%%%%%% Insert the texts which can accomdate on firstpage in the tag "fmtext" %%%%%

\section{Introduction}
 Conformal prediction \cite{gammerman1998learning, vgs} 
 {has become a powerful tool} for uncertainty quantification, 
 %has seen a surge of popularity in recent years,
 {with prominent applications in machine learning algorithms for regression and classification; see for example \cite{balasubramanian2014conformal} and \cite{angelopoulos2022gentle}. Conformal prediction is a family of algorithms that can be used to generate finite } %Recent applications include graph neural networks (for e.g., \cite{clarkson2023distribution, zargarbashi2023conformal})}\end{fmtext} 
%and time series forecasting, see, e.g., \cite{stankeviciute2021conformal}.

%%%%%%%%%%%%%%% End of first page %%%%%%%%%%%%%%%%%%%%%

\maketitle

\noindent
 sample valid prediction intervals or sets from a black-box machine learning model. %Conformal prediction may be thought of as a \say{wrapper} around a fitted model that uses a set of exchangeable held-out data to calibrate prediction sets. 
The possibly most used variant, %so called 
\textit{split-conformal} (explained in more detail in Section \ref{sec:split_cp}), also requires trivial computational overhead when compared to model fitting. %[stolen from CoS report] 
%Split-conformal prediction will be explained in more detail in Section \ref{sec:split_cp}.

% Will return to introduction later on when story of paper is clearer

% Talk about how 

%Something here about how conformal prediction works, including coverage. Foygel-Barber paper here. Total variation bounds. Classification and regression as main examples. Steal from the NAPS paper. Mention split conformal prediction as variant which will be explained in more detail in Section \ref{sec:split_cp}.

 % Amazingly, the predictive model does not even need be well specified for the conformal prediction  guarantees to hold;  however  the prediction intervals or sets may not be useful in this case. 

One aspect that %has not 
{has only recently} received %much
attention 
%to date % the literature 
is applying conformal prediction to data containing outliers. Although conformal prediction yields intervals with finite sample guarantees when the model is correctly specified, 
if outliers are not corrected, the conformal prediction intervals may not provide the coverage that the user expects.  
This paper addresses this issue. Our setting is that of a Huber-type mixture model (\cite{huber1964robust,huber1965robust}) of independent observations which are assumed to come from a distribution $\pi_1$, 
but are actually contaminated by a small number of observations which have distribution $\pi_2$. An observation from $\pi_1$ is called {\it clean}. A typical research question is then to predict the clean response,
%under this contamination model,
see for example \cite{chen2022online}. {This setting was first considered in \cite{barber2023conformal}; in \cite{sesia2023adaptive}, a more comprehensive treatment is given for label-conditional coverage and for marginal coverage. \cite{penso2025conformal} consider the particular setting of classification with uniform label noise.} 
%\gr{; more background on robust statistics can be found for example in \cite{huber2011robust}}.

Intuitively, if one knew which component of the mixture each data point was sampled from, one could simply calibrate the prediction set using only samples from the same mixture as the test point as these data points are exchangeable. We instead assume that it is unknown which data points are clean, and study the impact of contamination on the coverage and size
%width 
of the constructed sets.
In particular, we provide lower and upper bounds on the coverage under the assumption that the new data point is clean. {These  coverage bounds are designed to tend to 0 when the contamination proportion tends to 0.}
We provide a general robustness result regarding the construction of conformal prediction sets. Under the additional assumption that the contaminating distribution stochastically dominates the clean distribution, as may be the case in a regression setting, we derive an over-coverage guarantee, and we give a companion result for the situation that the clean distribution stochastically dominates the contaminating distribution.

Moreover, in a classification setting, we devise {\it Contamination Robust Conformal Prediction}, abbreviated CRCP, %a remedy for 
adjusting the prediction sets to obtain improved coverage guarantees. %We call this method {\it Contamination Robust Conformal Prediction}, abbreviated CRCP. 
Both on synthetic data and on a benchmark data containing real-world label noise, namely the CIFAR-10N data set by \cite{wei2022learning}, we find that while standard conformal prediction can give considerable over-coverage, CRCP not only ameliorates this issue but also produces considerably narrower prediction intervals. {Using a conference version of CRCP, \cite{penso2025conformal} draws a similar conclusion on other data sets.}

{Our work adds the following novelties to the literature.
\begin{itemize}[noitemsep, topsep=3pt,leftmargin=*]
    \item Our theoretical results recover the non-contaminated behaviour as   the contamination proportion tends to 0.
    %\item Our theoretical bounds improve on those from \cite{sesia2023adaptive} and do not require lower bounds on the number of samples from each label.
    \item {Compared to \cite{sesia2023adaptive}, our theoretical bounds are tighter in that they do not require lower bounds on the number of samples from each label. Relative to \cite{penso2025conformal}, our method applies in a more general setting, as their guarantees are restricted to discrete uniform label noise.}

    \item We employ our theoretical bounds to introduce {\it Contamination Robust Conformal Prediction (CRCP).} This method is based on tighter bounds and shows better performance than the Adaptive Conformal Classification with Noisy Label (ACNL) method from \cite{sesia2023adaptive} and it applies to a wider setting {than} the Noise-Aware Conformal Prediction (NACP) method from \cite{penso2025conformal}, for which the theoretical guarantees require discrete uniform noise. 
 \item Instead of relying solely on the total variation distance, in which many guarantees in the literature are stated, we provide bounds in terms of the Kolmogorov--Smirnov distance %, the Wasserstein distance,
 and the squared Le Cam distance. {These metrics are often easier to estimate in practice and allow us to extend the range of theoretical guarantees available for contaminated conformal prediction.}
\end{itemize}}
  
The paper is structured as follows. Section \ref{sec:split_cp} gives background on split conformal prediction when  the calibration and the test data are exchangeable. In Section \ref{sec:cp_mixture}, split conformal prediction under data contamination is introduced; we give %it %Section \ref{sec:theory} 
%gives
theoretical results for coverage and robustness, and  refined results under stochastic dominance assumptions. In Section~\ref{sec:remedy}, CRCP as a remedy for adjusting over-coverage in a classification setting is %introduced 
{devised}.
%and discussed. 
Experiments are shown in Section \ref{sec:experiments}, with Subsection \ref{subsec:synthadata} illustrating CRCP on synthetic datasets and Subsection \ref{subsec:realdata} comparing CRCP and standard conformal prediction on the CIFAR-10N dataset from \cite{wei2022learning}.
%the results are illustrated by experiments. 
%Perhaps somewhat unusually, 
A detailed 
discussion of related results %, which draws on  results which are derived in this paper, 
is postponed to the concluding Section \ref{sec:discussion}.
Proofs {as well as experimental details} are found in the Appendix. 
The code for  
this paper is available at %this link  
\url{https://github.com/reinertdaniel/cp_under_data_contamination_refactor.git}.
% \footnote{adapted from \url{https://github.com/jase-clarkson/cp_under_data_contamination}}

\section{Split Conformal Prediction for Exchangeable Data} \label{sec:split_cp}
Here we lay out notation and briefly describe the split conformal prediction procedure in the exchangeable setting; see \cite{angelopoulos2022gentle} for an excellent {and} extensive introduction. Suppose we have access to a pre-fitted model $\hat{f}: \mathcal{X} \rightarrow \mathcal{Y}$, and a set of \textit{calibration} data points $Z_i = (X_i, Y_i)$, $i = 1, \dots, n$ that were not used to fit the model. The goal is to construct a $(1 -\alpha)$-probability prediction set for a test {data point} $Z_{n+1} = (X_{n+1}, Y_{n+1})$, where $\alpha$ is a user-specified desired level of coverage. Conformal prediction {uses} a \textit{score function} $S: \mathcal{X} \times \mathcal{Y} \rightarrow \mathbb{R}$ which {quantifies} the agreement between the model predictions and the targets. We assume %without loss of generality 
that the score function is \textit{negatively oriented}; a smaller score indicates a better fit. %In this sense, the score function is essentially a loss function. 
%The pre-fitted model is only needed to obtain the scores; once the scores are available, the pre-fitted model no longer plays any role.

  The split conformal prediction procedure proceeds as follows. First, one computes the score for each calibration {data point} $S_i = S(X_i, Y_i)$. For a desired coverage of at least $1-\alpha$, the crucial step is to  estimate the prediction set boundary $\hat{q}$ as 
\begin{equation}
\label{eq:hatq}
    \hat{q} = Q_{1-\alpha} \Big( \frac{1}{n+1} \sum_{i=1}^n \delta_{S_i} + \frac{1}{n+1} \delta_{+\infty}\Big), 
\end{equation}
where for a probability measure $\mu$ on $\mathbb{R}$, $Q_{1-\alpha}(\mu) = \inf \{x: \mu ((-\infty, x]) \ge 1 - \alpha\}$; here $\delta_x$ is the delta-measure, representing {a} point mass at $x$.
{The quantity $\hat{q}$ in \eqref{eq:hatq} can be seen as an empirical quantile.} The procedure 
is equivalent to taking the $i$th order statistic of the scores $\{S_1, \ldots, S_n\}$, given by $\hat{q} = S_{(i)}= S_{i:n}$, where $i = \lceil (1-\alpha)(n+1) \rceil$ (ties are broken at random); we use the notation $S_{(i)}$ for simplicity.
Here we recall that the order statistics of $S_1, \ldots, S_n$ are $S_{(1)} \le S_{(2)} \le \cdots \le S_{(n)}$. Finally, a prediction set is constructed as 
\begin{equation} \label{eq:cp_set}
\widehat{C}_{n}\left(X_{n+1}\right)=\left\{y \in \mathcal{Y}: S\left(X_{n+1}, y\right) \leq \hat{q}\right\}.
\end{equation}
If the calibration and test data are exchangeable, 
then we have the coverage guarantee
\begin{equation}\label{eq:lower_upper}
    1-\alpha \leq \mathbb{P}\left(Y_{n+1} \in \widehat{C}_n\left(X_{ n+1 }\right)\right) \leq 1-\alpha+\frac{1}{n+1}.
\end{equation}
{The probability in \eqref{eq:lower_upper} is taken over the joint distribution of all data $(X_i, Y_i), i=1, \ldots, n+1$.} It is important to note that the coverage guarantee provided by conformal prediction is only \textit{marginal}; {for instance, if $T_{n+1}$ is a random variable that depends on $X_{n+1}$, then the guarantee in \eqref{eq:lower_upper} does not apply to} $\mathbb{P}\left(Y_{n+1} \in \widehat{C}_n\left(X_{ n+1 }\right) \mid {T}_{n+1}=x \right ).$ %We will revisit this point in the following section when discussing data contamination.

\section{Split Conformal Prediction under Data Contamination} \label{sec:cp_mixture}
\label{sec:theory} 
% Problem setup - lay out the model 
% Might be better to introduce mixture on the data rather than the scores.

%{A standard model for data contamination is the Huber $\varepsilon$-contamination model (\cite{huber1964robust,huber1965robust}), %huber2004robust}
%as follows. Let $\varepsilon \in [0, 1)$.}
%Suppose that the calibration data are sampled i.i.d from a mixture model \begin{equation}\label{model}
 %   Z_i = (X_i, Y_i) \sim (1-\varepsilon) \pi_1 + \varepsilon \pi_2, 
%\end{equation} 
%where $\pi_1, \pi_2$ are two distribution functions over %$\mathcal{X} \times \mathcal{Y}$. 
{A standard framework for modeling data contamination is the Huber $\varepsilon$-contamination model (\cite{huber1964robust,huber1965robust}). 
Let $\varepsilon \in [0,{1/2]}$. We assume that the calibration data are sampled i.i.d. from the mixture distribution
\begin{equation}\label{model}
    Z_i = (X_i, Y_i) \sim (1-\varepsilon)\,\pi_1 + \varepsilon\,\pi_2,
\end{equation}
where $\pi_1$ and $\pi_2$ are probability distributions on $\mathcal{X} \times \mathcal{Y}$.}
In this model,  %can be interpreted as
a (usually small) fraction $\varepsilon$ of the {data points} %being
{are} outliers with a different distribution from the bulk. 
%{Accordingly, 
Then the induced distribution of the scores is also a mixture:
$S_i = S(X_i, Y_i) \sim \tilde{\Pi},$
where {$\tilde{\Pi} = (1-\varepsilon)\Pi_1 + \varepsilon \Pi_2$ denotes the score distribution, with cumulative distribution function (cdf)} 
$\tilde{F} = (1-\varepsilon)F_1 + \varepsilon F_2.$
Here, {$\Pi_j$ denotes the distribution of the scores $S(X_i,Y_i)$ when $(X_i,Y_i) \sim \pi_j$, and $F_j$ its corresponding cdf, for $j=1,2$.} 
{Throughout, we assume that contamination occurs independently across observations.}

\begin{comment}

Under this model, the scores $S(X_i, Y_i)$ of the data are also distributed as a mixture, %and we write 
%\begin{align*}
   $ S_i = S(X_i, Y_i) \sim 
    \tilde{{\Pi}}, $
  where   
 % \end{align*}  
$\tilde{\Pi}$ has  cumulative distribution function (cdf)  
 $
    \tilde{F} = (1 - \varepsilon)F_1 + \varepsilon F_2
$ 
with  $F_1, F_2$ cumulative distribution functions over the scores computed from each mixture component. {We let ${\Pi}_j$ denote the distribution of the scores $S(X_i,Y_\gr{i})$ when 
$(X_i,Y_i) \sim \pi_j $, for $j=1, 2$.} In this setup, the contamination occurs independently of the previous observations.
%and $\varepsilon \in [0, 1)$. %This is also known as a Huber contamination model \cite{huber2004robust} and can be interpreted as a small fraction $\varepsilon$ of the scores being outliers with a different distribution from the bulk. 
\end{comment}

We denote the quantile in 
Equation 
\eqref{eq:hatq} estimated over the corrupted calibration data as $\tilde{q}$. By de Finetti's Theorem, mixture models are exchangeable and so conformal prediction provides coverage on future test points; however this coverage is marginal, and in particular only holds for future test points sampled \textit{from the contaminated mixture distribution} so that we only have
%i.e.
$ \mathbb{P}(\tilde{S}_{n+1} \leq \tilde{q}) \geq 1 - \alpha
$ 
for $\tilde{S}_{n+1} \sim \Tilde{\Pi}$. Instead, 
%this paper,
we are interested in the setting where the test point is assumed to be ``clean'', i.e. that $S_{n+1} \sim \Pi_1$. %We shall provide bounds on 
{T}he coverage over future clean test points {is}
\begin{equation}\label{eq:P1} \mathbb{P}\left(Y_{n+1} \in \widehat{C}_n\left(X_{ n+1 }\right) | Z_{n+1} \sim \pi_1 \right) = \mathbb{P} \left(S_{n+1} \leq \tilde{q} | S_{n+1} \sim {\Pi_1}\right) =: \mathbb{P}_1 (S_{n+1} \leq \tilde{q} ) .\end{equation}
{Here $\mathbb{P}_1 \ne P_1$  abbreviates the probability which takes the randomness in $\tilde{q}$ into account but conditions on the observation $S_{n+1}$ being clean.}
As the {random} quantile $\tilde{q}$ was estimated using contaminated data, the coverage guarantee given in Equation 
\eqref{eq:lower_upper} no longer holds. While we focus on coverage over future clean test points, \cite{sesia2023adaptive} obtain related results for label-conditional coverage and for marginal coverage.

We 
provide bounds on the coverage obtained under the corruption model \eqref{model}. 
We %
then study the average change in prediction set size as a result of data contamination,
and 
provide some remarks on when one might expect over- or under-coverage, illustrated by a regression example. The last part of this section studies classification under label noise.

Our results employ the
following distances. 
For two distributions $Q_1$ and $Q_2$ on $\mathbb{R}$ with cdf's $G_1$ and $G_2$ the Kolmogorov-Smirnov distance between $Q_1$ and $Q_2$ is 
\begin{equation} \label{KS} 
    d_{KS}(Q_1, Q_2) = \sup_x | G_1(x) - G_2(x)|.
\end{equation}
Next we use the notation
$\pi_i (x)\, dx= dG_i(x) $, for $i=1, 2,$; if $G_1$ and $G_2$ are continuous, then $\pi_i$  is just the probability density function of $G_i$, $i=1, 2.$ 
The   {squared Le Cam distance}  between $G_1$ and $G_2$ is
\begin{equation}
    \label{SQLC}
    \Delta(G_1, G_2) \;:=\; \tfrac{1}{2}\int_{\mathbb R} \frac{(\pi_2(x)-\pi_1(x))^2}{\pi_1(x)+\pi_2(x)}\,dx.
\end{equation}
This is a classical measure of distributional difference (see, e.g., \cite{vajda2009metric,taneja2013seven}); it satisfies
    $0 \le \Delta(F_1, F_2) \le 1.$ 
Finally, with $G_1$ and $G_2$ having inverses $G_1^{-1}$ and  $G_2^{-1}$, the  Wasserstein-$p$ distance between $Q_1$ and $Q_2$ is 
\begin{equation}\label{Wass}
{    W_p(Q_1, Q_2) =
   \Big( \int_{0}^1 | G_1^{-1}(q) - G_2^{-1}(q)|^p dq \Big)^{\frac1p}.}
   %\mbox{ and }
%    W_1(Q_1, Q_2) =
%\int_{0}^1 | G_1 (q) - G_2(q)| dq }
\end{equation}

\subsection{Coverage}
 
Lemma \ref{lemma:lb} captures the intuition that the difference in coverage increases with the contamination fraction $\varepsilon$ and the magnitude of the contaminations, measured by $d_{KS}(F_1, F_2)$. 

\begin{lemma} \label{lemma:lb}
    Under  the mixture model, \eqref{model}
    {when $(X_{n+1}, Y_{n+1}) \sim \pi_1$}, we have
    \begin{equation} \label{eq:mixture_lb}
     - \varepsilon \, \E[F_2(\tilde{q}) - F_1(\tilde{q})]
     \leq \mathbb{P}_1\Big( Y_{n+1} \in \widehat{C}_n(X_{n+1}) \Big) -  (1-\alpha) \leq \frac{1}{n+1} + \varepsilon \, \E[F_1(\tilde{q}) - F_2(\tilde{q})] 
    \end{equation}
    and 
    $ - \varepsilon \, d_{KS}(\Pi_1, \Pi_2) \leq \mathbb{P}_1\Big( Y_{n+1} \in \widehat{C}_n(X_{n+1}) \Big) - (1-\alpha) \leq  \frac{1}{n+1} + \varepsilon \, d_{KS}(\Pi_1, \Pi_2).
    $
\end{lemma}
\begin{proof} Let $\tilde{S}_{n+1}$ be %a random variable 
sampled from the mixture distribution $\tilde{S}_{n+1} \sim \tilde{P}$. 
We first derive the claimed lower bounds. 
By the lower bound given in Equation \eqref{eq:lower_upper} applied to the mixed data, we have 
    \begin{equation} \label{eq:mixed_guarantee}
    \mathbb{P}(\tilde{S}_{n+1} \leq \tilde{q}) \, \geq  \, 1 - \alpha.
    \end{equation}
Let $P$ be the distribution of the random quantity $\tilde{q}$. Then conditioning yields 
\begin{align} 
\mathbb{P}(\tilde{S}_{n+1} \leq \tilde{q}) &= \int_{\mathbb{R}}\mathbb{P}(\tilde{S}_{n+1} \leq \tilde{q} | \tilde{q} = q) dP(q) 
= \int_{\mathbb{R}} \{ (1-\varepsilon)F_1(q) + \varepsilon F_2(q) \} dP(q),  \label{eq:mixture_int}
\end{align}
where the last step in \eqref{eq:mixture_int} 
follows from the independence of $S_{n+1}$ and $(S_{1}, \ldots, S_n)$.
Substituting  \eqref{eq:mixture_int} into \eqref{eq:mixed_guarantee} and re-arranging gives
$\int_{\mathbb{R}}F_1(q)dP(q) \geq (1-\alpha) - \varepsilon \int_{\mathbb{R}} \{  F_2(q) - F_1(q) \} dP(q)$. 
Finally, un-doing the conditioning gives 
\begin{align}
    \int_{\mathbb{R}} F_1(q) dP(q) = \int_{\mathbb{R}} \mathbb{P}(S_{n+1} \leq q) dP(q) &= \int_{\mathbb{R}} \mathbb{P}(S_{n+1} \leq \tilde{q} | \tilde{q} = q) dP(q) \nonumber% \\
   % &
    = \mathbb{P}_1(S_{n+1} \leq \tilde{q}), \nonumber 
\end{align}
with $\mathbb{P}_1$ as in Equation \eqref{eq:P1},  
proving the lower bound in 
\eqref{eq:mixture_lb}. The lower bound in Kolmogorov-Smirnov distance follows directly from \eqref{KS} applied to the lower bound in the inequality  \eqref{eq:mixture_lb}. 
The upper bounds are proved in an analogous fashion, applied to the upper bound 
in  \eqref{eq:lower_upper}.
\end{proof}

The difference $\E [ F_2 (\qtilde) - F_1 (\qtilde)]$ will not generally be small (but can be bounded by 1); it is the factor $\varepsilon$ that can render the bounds in Lemma \ref{lemma:lb} small. 
Hence in this paper we focus on the behaviour of the bounds as $\varepsilon \rightarrow 0$.

\begin{remark}
 Appendix C in \cite{barber2023conformal}  provides a lower bound 
for the coverage probability which is comparable to the lower bound in Lemma \ref{lemma:lb}; in the Huber contamination setting, the bound in \cite{barber2023conformal}  improves on the lower bound in  Lemma \ref{lemma:lb} when $d_{KS}(\Pi_1, \Pi_2)$ is large or, similarly, when $\E[F_1(\tilde{q}) - F_2(\tilde{q})] $ is large, where large means larger than $\alpha/(1-\alpha)$. 
{Moreover, in \cite{barber2023conformal}, a coverage bound with a multiplicative factor $\alpha/(1-\varepsilon)$ is obtained, and an additive version is briefly mentioned {but not detailed}. Related bounds are also developed in \cite{sesia2023adaptive} for label-conditional coverage and for marginal coverage in the case of classification under label noise. 
{In \cite{sesia2023adaptive}, further assumptions on the marginal noise distributions are made such as having bounded probability density functions; moreover, the obtained bounds are typically only informative when the number of observations in each class is bounded away from 0.} %using a more specialized formulation. 
These results are complementary to our simpler additive bound under general calibration contamination. 
}
\end{remark}
 Next we provide an alternative bound via the squared Le Cam distance \eqref{SQLC}. 
    Here we use the notation
$\pi_i (x)\, dx= dF_i(x) $, for $i=1, 2,$ and $\tilde{\pi} = (1-\varepsilon) \pi_1 + \varepsilon \pi_2.$ 
We first provide a representation of $F_2 - F_1$ which
%The key point is that Stein's method 
isolates %the
{a} signed weight 
$w_q$ whose variance is explicit and bounded by the variance of a 
Bernoulli random variable with parameter $\tilde F(q)$.
The proof is straightforward and is found in the Appendix.

\begin{lemma} \label{lem:rep}
For any $q\in\mathbb R$, 
\[
F_2(q)-F_1(q) 
= \E_{\tilde \pi}\!\big[ w_q(\tilde T)\,h(\tilde T) \big],
\]
where $ \tilde T\sim \tilde \pi; 
 w_q(x)=(1-\tilde F(q))\,\mathbf 1\{x\le q\}-\tilde F(q)\,\mathbf 1\{x>q\}$, and $h(x)=({\pi_2(x)-\pi_1(x)})/{\tilde \pi(x)}$. 
Moreover,
$
\E_{\tilde \pi}[\,w_q(\tilde T)^2\,]
= \tilde F(q)(1-\tilde F(q)).
$
\end{lemma}
\begin{remark}
    The representation in Lemma \ref{lem:rep} is natural using Stein's method; the function $w_q (x)$ is related to the solution of the so-called Stein equation for 
    $\Ftilde$; see for example \cite{ernst2020first}.
\end{remark}

%This lemma provides  
%a novel bound on the discrepancy between $F_2$ and $F_1$ as a covariance-like quantity with an 
%explicit variance term. 

\begin{comment}
    \begin{proposition}\label{prop:sharper} 
For any $q\in\mathbb R$,
\[
\varepsilon\,|F_2(q)-F_1(q)| 
\;\le\;
\sqrt{\varepsilon}\,\sqrt{\Delta(F_1,F_2)}\,
\sqrt{\tilde F(q)\,\bigl(1-\tilde F(q)\bigr)}.
\]
In particular, for the split--conformal threshold $\tilde q$,
\[
\varepsilon\,\E\bigl|F_2(\tilde q)-F_1(\tilde q)\bigr|
\;\le\;
\sqrt{\varepsilon}\,\sqrt{\Delta(F_1,F_2)}\,
\E\!\left[\sqrt{\tilde F(\tilde q)\,\bigl(1-\tilde F(\tilde q)\bigr)}\,\right],
\]
where the expectation is taken over the randomness of $\tilde q$.
\end{proposition}

\begin{proof}
By Lemma~\ref{lem:rep} and the Cauchy--Schwarz inequality,
\[
|F_2(q)-F_1(q)| 
\;\le\; 
\|w_q\|_{L^2(\tilde \pi)}\,\|h\|_{L^2(\tilde \pi)}.
\]
The first factor satisfies 
\(\|w_q\|_{L^2(\tilde \pi)}^2=\tilde F(q)\bigl(1-\tilde F(q)\bigr)\).  
For the second factor,
\[
\|h\|_{L^2(\tilde \pi)}^2
= \int \frac{(\pi_2-\pi_1)^2}{\tilde \pi}
\;\le\; \frac{1}{\min\{\varepsilon,1-\varepsilon\}}\,
\Delta(F_1,F_2).
\]
If $\varepsilon\le 1/2$, then $\min\{\varepsilon,1-\varepsilon\}=\varepsilon$, and hence
\[
\|h\|_{L^2(\tilde \pi)}^2 \;\le\; \frac{1}{\varepsilon}\,\Delta(F_1,F_2).
\]
Therefore,
\[
\varepsilon\,|F_2(q)-F_1(q)| 
\;\le\;
\sqrt{\varepsilon}\,\sqrt{\Delta(F_1,F_2)}\,
\sqrt{\tilde F(q)\bigl(1-\tilde F(q)\bigr)}.
\]
Finally, taking the expectation over the randomness of $\tilde q$ yields the second claim.
\end{proof}

\end{comment}

\begin{proposition}\label{prop:sharper} 
Let $0 < \varepsilon <  1$ and set $\varepsilon_{\min} = \min (\varepsilon, 1-\varepsilon). $
For any $q\in\mathbb R$, 
\[
\varepsilon_{\min}\,|F_2(q)-F_1(q)| 
\;\le\;
\sqrt{2\varepsilon_{\min}}\,\sqrt{\Delta(F_1,F_2)}\,
\sqrt{\tilde F(q)\,\bigl(1-\tilde F(q)\bigr)}.
\]
In particular, for the {(random)} split--conformal threshold $\tilde q$,
\[
\varepsilon\,\E\bigl|F_2(\tilde q)-F_1(\tilde q)\bigr|
\;\le\;
\sqrt{2\varepsilon_{\min}}\,\sqrt{\Delta(F_1,F_2)}\,
\E\!\left[\sqrt{\tilde F(\tilde q)\,\bigl(1-\tilde F(\tilde q)\bigr)}\,\right].
\]
%where the expectation is taken over the randomness of $\tilde q$.
\end{proposition}

\begin{proof}
By Lemma~\ref{lem:rep} and the Cauchy--Schwarz inequality,
$
|F_2(q)-F_1(q)| 
\;\le\; 
\|w_q\|_{L^2(\tilde \pi)}\,\|h\|_{L^2(\tilde \pi)}.
$ 
The first factor satisfies 
\(\|w_q\|_{L^2(\tilde \pi)}^2=\tilde F(q)\bigl(1-\tilde F(q)\bigr)\).  
For the second factor,
\[
\|h\|_{L^2(\tilde \pi)}^2
= \int_{\R}  \frac{(\pi_2 (x) -\pi_1 (x) )^2}{\tilde \pi (x) }\, dx 
\;\le\;% {\color{red}{2}} \frac{1}{\min\{\varepsilon,1-\varepsilon\}}\,
%\Delta(F_1,F_2)
%\le 
{\frac{2}{\varepsilon_{\min}} \Delta(F_1,F_2).}
\]
%\]
%If $\varepsilon\le 1/2$, then $\min\{\varepsilon,1-\varepsilon\}=\varepsilon$, and hence
%\[
%\|h\|_{L^2(\tilde \pi)}^2 \;\le\; \frac{1}{\varepsilon}\,\Delta(F_1,F_2).
%\]
Substituting the two bounds {and taking} the expectation over $\tilde q$ proves the second claim.
%yields
%\[
%\varepsilon\,|F_2(q)-F_1(q)| 
%\;\le\;
%\sqrt{\varepsilon}\,\sqrt{{\color{red}{2}}\Delta(F_1,F_2)}\,
%\sqrt{\tilde F(q)\bigl(1-\tilde F(q)\bigr)}.
%\]
%Taking the expectation over $\tilde q$ proves the second claim.
\end{proof}

\begin{remark}
In general the bounds in Lemma \ref{lemma:lb} are not observable, as $F_1$ and $F_2$ are unknown. However the trivial bound $|\E F_1 (\qtilde) - \E F_2  (\qtilde) | \le 1$ holds, giving as lower bound $(1 - \alpha) -\varepsilon$ and as upper bound $(1 - \alpha) + \frac{1}{n+1}  + \varepsilon.$
    Similarly, as the squared Le Cam distance is bounded by 1,  the bound in Proposition \ref{prop:sharper} can be evaluated by estimating only $\Ftilde$. Moreover, 
       since $\Ftilde(\qtilde)\le 1-\alpha+1/(n+1)$ and $1-\Ftilde(\qtilde)\le \alpha$, 
    {the bound in Proposition \ref{prop:sharper} is no larger than} 
   $
    \sqrt{2\varepsilon\,\alpha\bigl(1-\alpha+1/(n+1)\bigr)}.
   $ %{This observation} is informative in the sense that 
   {T}he resulting upper bound in Lemma \ref{lemma:lb} is less than {the trivial bound} 1 when 
   $ \varepsilon < {(\alpha - 1/(n+1))
   }/{(\alpha ( 1 - \alpha + 1/(n+1)))}.$ 
   % As $n\to\infty$ this expression is smaller than the asymptotic bound in Proposition~\ref{prop:quantilebound} whenever $\varepsilon > \alpha(1-\alpha)/2$, i.e.\ exactly in the regime where Proposition~\ref{prop:quantilebound} becomes uninformative. 
   %     Moreover, for the sharper bound itself to be informative as $n\to\infty$, one requires
    %\[
    %  \sqrt{2\varepsilon\,\alpha(1-\alpha)} \;\le\; \alpha 
    %  \quad \Longleftrightarrow \quad
   %   $\varepsilon \;\le\; \frac{\alpha}{2(1-\alpha)}$.
    %\]
 %   Since $\alpha(1-\alpha)/2 < \alpha/[2(1-\alpha)]$ for all $0<\alpha<1$, there exists a region of $\varepsilon$ where Proposition~\ref{prop:sharper} improves on Proposition~\ref{prop:quantilebound}.
 Moreover, this bound is smaller than the standard  bound $\varepsilon$ if $\varepsilon \geq \alpha ( 1 - \alpha + 1/{(n+1)})$. 
\end{remark}

\subsection{Robustness}

Next we compare the expected size of prediction sets constructed using a clean calibration set drawn from $F_1$, and a corrupted sample drawn from the mixture. Here we exploit that $\tilde{q}$ is a random quantile {that can be analyzed} using tools from order statistics. Suppose that $S_1, \ldots, S_n$ are i.i.d.\,{absolutely-continuous} random variables with distribution function $F$ and density $f$. Let $S_{(i)}$ be the $i^{\text{th}}$ order statistic of this sample. Then, {see \cite[p.~108]{fcios}},   
\begin{equation} \label{eq:os_mean}
    \mathbb{E}[S_{(i)}] = \frac{1}{B(i, n-i+1)} \int_{\mathbb{R}} x (F(x))^{i-1} (1 - F(x))^{n-i} f(x) dx, 
\end{equation}
where $B(\alpha, \beta) = \frac{\Gamma(\alpha)\Gamma(\beta)}{\Gamma(\alpha + \beta)}$ is the beta function. %Moreover, $F(S_{(i)}) =: U_{(i)} \sim \mathrm{Beta}(i, n-i+1)$, i.e., {has probability density function (pdf)} 
%\begin{equation}\label{eq:beta}
% $\frac{1}{B(i, n-i+1)} x^{i-1} (1 - x)^{n-i},$ with mean $\frac{i}{n+1}.$ 
%\end{equation}
%{i.e., $U_{(i)} \sim \mathrm{Beta}(i, n-i+1)$.}

In the following lemma, we use the {model} \eqref{model} 
and 
the Wasserstein distance from 
\eqref{Wass}. {We use the notion that a random variable $X$ is absolutely continuous if there is an integrable function $f$ such that  $\P(X \le a)  = \int_{-\infty}^a f(y) dy$ for all $a \in \R$.}
\begin{lemma} \label{lemma:w1_bound}
Let $S_1, \dots, S_n$ be scores sampled i.i.d.\,from $\Pi_1$, and let $\tilde{S}_1, \dots, \tilde{S}_n$ be sampled i.i.d.\,from $\tilde{\Pi}$. {Assume these scores are absolutely continuous random variables.} Define $i = \lceil (1-\alpha)(n+1) \rceil$ and let $S_{(i)}, \tilde{S}_{(i)}$ be the $i^{\text{th}}$ order statistics of the first and second samples, respectively. {Then for any $p\in[1,\infty]$,}
\begin{equation}
\big| \mathbb{E}[\tilde{S}_{(i)}] - \mathbb{E}[S_{(i)}] \big| \leq C(n, i)\, {W_p(\tilde{\Pi}, \Pi_1)}, 
\end{equation}
where 
$
C(n, i) \,=\, \sup_{t \in [0, 1]} \frac{t^{i-1}(1-t)^{n-i}}{B(i, n - i + 1)}
\,=\, \frac{1}{B(i, n-i+1)} \left ( \frac{i-1}{n-1} \right )^{i-1} \left ( \frac{n-i}{n-1} \right )^{n-i}.
$
{In particular, if $\tilde{\Pi}=(1-\varepsilon)\Pi_1+\varepsilon\Pi_2$, then for $p=1$}
\begin{align}
\big| \mathbb{E}[\tilde{S}_{(i)}] - \mathbb{E}[S_{(i)}] \big| \leq \varepsilon\, C(n, i)\, {W_1(\Pi_1, \Pi_2)}.
 \label{eq:wass1}
\end{align}
\end{lemma}

\begin{proof}
    The substitution $x=F^{-1}(u)$ in \eqref{eq:os_mean} {(and analogously for $\tilde{F}$) %, and using \eqref{eq:beta}},
    yields
\[
 \mathbb{E}[S_{(i)}] = \int^1_0 F^{-1}(u)\, \frac{u^{i-1}(1-u)^{n-i}}{B(i, n - i + 1)}\, du
\quad \text{ and } \quad 
 \mathbb{E}[\tilde{S}_{(i)}] = \int^1_0 \tilde{F}^{-1}(u)\, \frac{u^{i-1}(1-u)^{n-i}}{B(i, n - i + 1)}\, du.
\]
Therefore,
\begin{align*}
  &  \big | \mathbb{E}[\tilde{S}_{(i)}] - \mathbb{E}[S_{(i)}] \big | 
    \leq   \int_{0}^1 \Big| \tilde{F}^{-1}(u) - F_1^{-1}(u) \Big|\,
            \frac{u^{i-1}(1-u)^{n-i}}{B(i, n - i + 1)} \, du \\
    &\leq {\left\| \frac{u^{i-1}(1-u)^{n-i}}{B(i, n-i+1)} \right\|_{L^p([0,1])}}
          \left (\int_{0}^1 \left |\tilde{F}^{-1}(u) - F_1^{-1}(u) \right |^p du \right )^{\frac{1}{p}} \stepcounter{equation} \tag{\theequation}\label{eq:holder} = {C(n,i)}\, {W_q(\tilde{\Pi}, \Pi_1),}
\end{align*} 
{where we applied H\"older's inequality with conjugate exponents $(p,q)$}.}
%and then used that the Beta weight is bounded by $C(n,i)$ on $[0,1]$.}
For $q=1$, the 1D quantile representation gives $W_1(\tilde{\Pi}, \Pi_1)=\int_0^1 |\tilde{F}^{-1}(u)-F_1^{-1}(u)|\,du$. If $\tilde{\Pi}=(1-\varepsilon)\Pi_1+\varepsilon\Pi_2$, a coupling that matches the $(1-\varepsilon)$-mass of $\Pi_1$ to itself and transports only the $\varepsilon$-mass from $\Pi_2$ to $\Pi_1$ yields $W_1(\tilde{\Pi}, \Pi_1)\le \varepsilon\,W_1(\Pi_2,\Pi_1)$, which implies the stated particular case.
\end{proof}

\begin{remark}
{For %general 
$p\in[1,\infty)$ and %a mixture 
$\tilde{\Pi}=(1-\varepsilon)\Pi_1+\varepsilon\Pi_2$, a standard coupling together with convexity of $W_p$ yields}
${W_p(\tilde{\Pi},\Pi_1)\;\le\;\varepsilon^{1/p}\,W_p(\Pi_1,\Pi_2).}$ 
{Thus, Lemma \ref{lemma:w1_bound} implies that}
${\big|\mathbb{E}[\tilde{S}_{(i)}]-\mathbb{E}[S_{(i)}]\big|\;\le\;C(n,i)\,\varepsilon^{1/p}\,W_p(\Pi_1,\Pi_2),}
$
{which for $p=1$ reduces to \eqref{eq:wass1}. %bound with the factor $\varepsilon$ given above. 
For $p=\infty$, one has $W_\infty(\tilde{\Pi},\Pi_1)\le W_\infty(\Pi_2,\Pi_1)$, so the same bound holds with $\varepsilon^{1/p}$ replaced by $1$.}
\end{remark}

Lemma \ref{lemma:w1_bound} thus gives a quantitative version of the intuition that when the mixture distributions are close, then so will be the quantiles of their scores. {These quantiles determine the prediction intervals, so if the quantiles are close then the prediction intervals will also be close.} In this sense, Lemma \ref{lemma:w1_bound} {serves as} a robustness guarantee.

The assumption in Lemma \ref{lemma:w1_bound} that the scores are absolutely continuous is mild;
%as 
any ties are broken at random %when they exist
by adding a small  continuous random variable with  uniform distribution, 
rendering the thus adjusted scores  absolutely continuous while retaining exchangeability.

\subsection{Theoretical Guarantees under Stochastic Dominance} \label{sec:over_under}
In the previous subsection we derived bounds on the coverage and {robustness} 
of the prediction sets constructed using contaminated data. We now introduce some conditions under which the Huber model \eqref{model} will always lead to over- {or under-coverage} under $S_{n+1} \sim \Pi_1$. 
In the following we will use a notion of stochastic ordering between random variables known as \textit{first-order stochastic dominance}. For two real-valued variables $X_1 \sim \Pi_1, X_2 \sim \Pi_2$, with corresponding {cdfs} $F_1$ and $F_2$, we say $X_2$ \textit{first-order stochastically dominates} $X_1$, and write $\Pi_1 \leq_{\mathrm{s.d.}} \Pi_2$, if 
\[
1 - F_1(x) \leq 1 - F_2(x) \quad \text{for all } x \in \mathbb{R}.
\] 
%{Equivalently, $F_1(x) \ge F_2(x)$ for all $x \in \mathbb{R}$.} %We write $\Pi_1 \leq_{\mathrm{s.d.}} \Pi_2$ to denote that $X_2 \sim \Pi_2$ first-order stochastically {dominates} $X_1 \sim \Pi_1$.

\paragraph{Over-coverage.} 
From the lower bound in \eqref{eq:mixture_lb}, if $\Pi_1 \leq_{s.d} \Pi_2$ , then  $\E[F_2(\tilde{q}) - F_1(\tilde{q})] \leq 0$ and %so
\begin{equation*}
    \P_1(S_{n+1} \leq \tilde{q}) \geq 1 - \alpha - \varepsilon \,  \E[F_2(\tilde{q}) - F_1(\tilde{q})] %>
    \geq 1 - \alpha.
\end{equation*}
In this case, conformal prediction still provides (conservative) coverage, but 
prediction set sizes may be inflated. 
We 
note that if $\Pi_1 \leq_{s.d} \Pi_2$ then we also have that $\Pi_1 \leq_{s.d} \tilde{\Pi}$, as
\begin{align}
    \tilde{F}(x) \leq F_1(x) &\iff (1-\varepsilon) F_1(x) + \varepsilon F_2(x) \leq F_1(x)
    \iff F_2(x) \leq  F_1(x).\label{eq:over_mixture_cond}
\end{align}

\paragraph{Under-Coverage.}
If instead the contaminating distribution $\Pi_2$ stochastically dominates the clean distribution $\Pi_1$, 
Inequality \eqref{eq:mixture_lb} in Lemma \ref{lemma:lb}  %\eqref{eq:mixture_ub_expectation} 
implies that for a $1-\alpha$ upper coverage bound we need a slightly stronger condition than $\Pi_2 \leq_{s.d} \Pi_1$, namely that 
\begin{equation} \label{eq:undercvg_cond}
    F_1(x) - F_2(x) \leq  - \frac{1}{\varepsilon (n+1)} \quad \text{for all }x\in \R.
\end{equation}
In this case by %Inequality 
\eqref{eq:mixture_lb} %in {Lemma \ref{lemma:lb}
 %Equation \eqref{eq:mixture_ub_expectation} 
we have
\begin{align*}
    \P_1(S_{n+1} \leq \tilde{q}) &\leq 
    (1 - \alpha) + \frac{1}{n+1} + \varepsilon \, \E[F_2(\tilde{q}) - F_1(\tilde{q})] %\\
   % &
    \le (1 - \alpha) + \frac{1}{n+1} - \frac{\varepsilon}{\varepsilon(n+1)} = 1- \alpha.
\end{align*}

 \begin{example}[A Regression Example]
\label{ex:regression} 
Consider the %following 
regression model 
\begin{align} \label{regmodel}
    Y &= f(X) + E, \text{ with } 
    E \sim (1-\varepsilon) \, \mathcal{N}(0, \sigma_1^2) + \varepsilon \, \mathcal{N}(0, \sigma_2^2), %nonumber 
\end{align}
where $\mathcal{N}(\mu, \sigma^2)$ denotes a {normally distributed} random variable with mean $\mu$ and variance $\sigma^2$, and $f: \mathcal{X} \rightarrow \R$ is {a measurable} %an arbitrary 
regression function. 
Assume that %the 
{a} forecaster has access to the oracle model, i.e. $\hat{f} = f$, and uses the absolute residual score function $S(X, Y) = |Y - \hat{f}(X)|$. Then the scores from each mixture component are distributed as a half normal distribution, {which has cdf}  
\begin{equation}
    F_{i}(x) = \operatorname{erf}\left(\frac{x}{\sqrt{2}\sigma_i}\right), \quad x \geq 0 \label{eq:undercvg_mixture_cond}
\end{equation}
{for $i=1,2$.} The error function  $\operatorname{erf}$ is increasing and so the clean distribution {is stochastically dominated by the contaminated one (hence coverage is conservative)} if $\sigma_1 \leq \sigma_2$. 

%{Although one} %we
%may expect 
%that $\sigma_1 > \sigma_2$ would imply under-coverage, %but
%this is not the case. In particular, 
{However,} if $\sigma_1 > \sigma_2$, then {for $x > 0$,}
\begin{align*}
    F_1(x) - F_2(x) &= \frac{2}{\pi}\int^{\frac{x}{\sqrt{2}\sigma_1}}_{\frac{x}{\sqrt{2}\sigma_2} } e^{-t^2} dt % =
%- \frac{2}{\pi}\int^{\frac{x}{\sqrt{2}\sigma_2}}_{\frac{x}{\sqrt{2}\sigma_1} } e^{-t^2} dt %\\
     \leq 
      -  \frac{\sqrt{2}x}{\pi}\left( \frac{1}{\sigma_2} - \frac{1}{\sigma_1} \right ) {e^{-\frac{x^2}{ 2 \sigma_1^2}}}.
\end{align*}
{Comparing with \eqref{eq:undercvg_cond}, we see that under-coverage requires the condition to hold for all $x \ge 0$:}
\begin{equation} \label{eq:regr_ucvg}
\frac{\sqrt{2}x}{\pi}\left ( \frac{1}{\sigma_2} - \frac{1}{\sigma_1} \right ) {e^{-\frac{x^2}{ 2 \sigma_1^2}}}
\geq \frac{1}{\varepsilon(n+1)}. \end{equation}
{As $x > 0$ can be chosen to violate \eqref{eq:regr_ucvg},}
$\sigma_1 > \sigma_2$ is not sufficient {to guarantee} under-coverage.
%\end{example}

To demonstrate the sensitivity of coverage to the noise parameters $\sigma_2$ and $\varepsilon$ we perform the following experiment using the model \eqref{regmodel} with  %described in Example \ref{ex:regression};
%we take 
$f(X) = \beta^T X$
and $\sigma_1=1$, while varying $\varepsilon$ and $\sigma_2$. %For the experiment 
We choose the  parameter vector
$\beta \in \R^p$ as a standard multivariate Gaussian vector and keep this $\beta$ fixed.
%and %set
%\sigma_1 = 1.0$. 
For each choice of $\sigma_2$ and $\varepsilon$, we draw 1000 samples from this model for training and test. We fit a linear regression model on the training set and use the fitted model to calibrate conformal prediction with the absolute residual score function. We then draw 1000 clean  samples for testing (i.e. with $\varepsilon=0$ in %Equation
\eqref{regmodel}), construct prediction sets for each point and record the mean coverage. We repeat this experiment 100 times for a range of choices of $\varepsilon$ and $\sigma_2$
and plot the mean and standard deviation of the coverage over these repetitions in Figure \ref{fig:coverage_abl}. In the left panel we see the transition between over- and under-coverage around $\sigma_1 = \sigma_2 = 1.0$. 
 \begin{figure}[t]
     \centering
\includegraphics[width=0.8\textwidth]{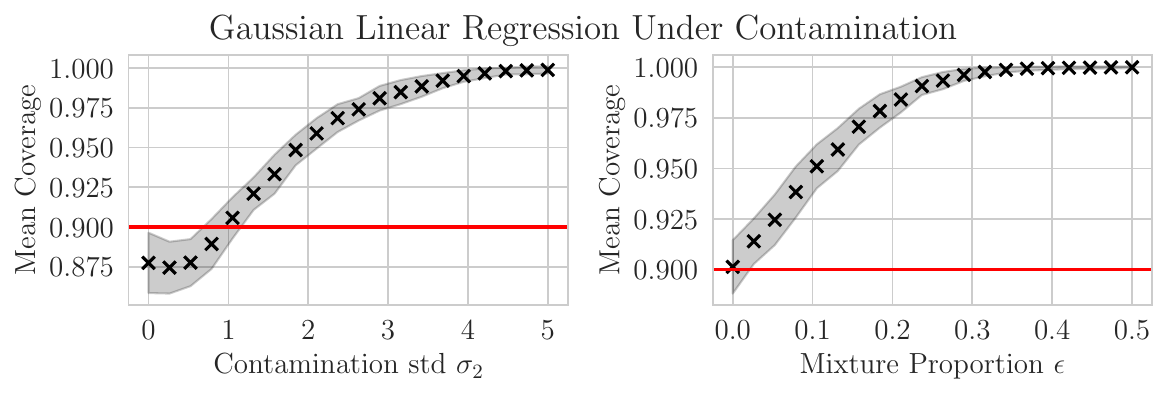}
     \caption{The mean and standard deviation of the coverage obtained over 100 repetitions of the regression experiment while varying $\varepsilon$ and $\sigma_2$. %different parameters. 
     The mean coverage is marked by crosses; {shaded regions} {indicate} one standard deviation of the coverage; the  straight horizontal line (in red) is  the desired coverage level $0.9$. Left: we vary the standard deviation of the corruption $\sigma_2$ from 0 to 5, keeping $\varepsilon = 0.2$. Right: we vary the mixing proportion $\varepsilon$ from $0$ to $0.5$, keeping $\sigma_2=3.0$.}
    \label{fig:coverage_abl}
 \vspace{-.5cm}
 \end{figure}
\end{example}

%\paragraph{Discussion}
%\vspace{-1cm}
The condition for over-coverage given in Equation \eqref{eq:over_mixture_cond} is weaker than the condition for under-coverage given in Equation \eqref{eq:undercvg_mixture_cond}; {over-coverage requires only stochastic dominance, i.e.\ $\Pi_1 \leq_{s.d} \Pi_2$, whereas under-coverage requires stochastic dominance with an additional margin that depends on both the sample size and the mixture proportion}. This finding suggests that conformal prediction possesses some \say{inbuilt} robustness to large outliers,  {since it continues to provide valid coverage (albeit with inflated prediction set sizes) regardless of the corruption level}, but there is no such guarantee when the outliers are small. %{Next we turn our attention to {classification} tasks.}

 \subsection{Classification under Label Noise} \label{sec:classification_overcoverage}

Suppose now the prediction set is constructed for a classification problem such that the targets take a discrete set of values $\mathcal{Y} = [K] = \lbrace 1, \dots, K \rbrace$. We write the generative model for the calibration data as $X_i \sim F_X$, and $Y_i \sim F_{Y|X}$, and assume that labels are corrupted with probability $\varepsilon \in (0, \frac{1}{2})
$, independently of the conditional distribution $X|Y$; we denote a sample from the corrupting distribution as ${Y}_i^c$. Then the observed class label $\tilde{Y}_i$ equals a draw $Y_i$ from $F_{Y|X}$ with probability $1-\varepsilon$, and %equals
a draw $Y_i^c$ from the corrupting distribution with probability $\varepsilon$. {By chance the draws $Y_i$ and $Y_i^c$ may coincide, so the label appears uncorrupted even when noise is applied.} 
{This setup is an analogue of the Huber contamination model applied to the label space: with probability $1-\varepsilon$ the label is drawn from the clean conditional distribution, while with probability $\varepsilon$ it is replaced by a corrupted label.}
{Such a model is a randomised response procedure  used for privacy protection; see, for example, \cite{nayak2016concise} and \cite{penso2025privacy}.}

  Let 
 $
     P_i = \Pi_1(i) = \P(Y=i), $ and $ \tilde{P}_i = \tilde{P}_i(\varepsilon)= \tilde{\Pi}(i) = \P(\tilde{Y}=i), 
$ 
for $i \in [K]$, be the marginal label probabilities. Let $P \in \R^{K \times K}$ be the matrix with entries
%\begin{equation*}
    $ P_{ji} = P_{ji}(\varepsilon)= \P(Y = j {\mid} \tilde{Y} = i)   $; 
%\end{equation*}
%for $i, j \in [K]$; 
we assume that $P$ is invertible {(in particular, $\tilde P_i>0$ for all $i$)}.  
Setting
%\begin{equation*}
    $\tilde{P}_{ij} = \tilde{P}_{ij}(\varepsilon)= \P(\tilde{Y} = i {\mid} Y=j) $, we have 
%\end{equation*}
\begin{equation} \label{eq:P_ji}
   \tilde{P}_i  P_{ji} = \tilde{P}_{ij} P_j   
\end{equation}
 by Bayes' rule.  
In what follows we suppress the argument $\varepsilon$ {for brevity}. Finally, we define 
\begin{align*}
    F_1(q; i,j) &= \P(S(X, i) \leq q {\mid} Y = j) \text{ and } 
    \tilde{F}(q; i,j) = \P(S(X, i) \leq q {\mid} \tilde{Y} = j) ; \quad \quad q \in \R, \quad i, j \in [K]
\end{align*}
%for $q \in \R$, $i, j \in [K]$ 
as the {cdfs of the} conditional distribution of the score assigned to label $i$ given that the true or observed noisy label is equal to $j$. 
{In analogy to stochastic dominance} %In contrast to regression problems, in classification problems there is no natural ordering which could give rise to stochastic dominance arguments. Instead, 
we use the condition 
\begin{equation}
    \label{condition:over}
    \max_{c: c \neq i} \, \P(S(X, c) \leq q {\mid} Y = i) \leq \P(S(X, i) \leq q {\mid} Y=i), \text{ for all }q \in \R, \;  i \in [K]. 
\end{equation} 
Condition \eqref{condition:over} %is natural;  due to the negative orientation of the score,  it is satisfied 
%for $\varepsilon=0$    if 
% $S(x, c) \geq S(x, i)$ when $i$ is the true label for $x$ and $c$ is any other label. %{This follows from the negative orientation of the score.} 
% Moreover, 
%It %Condition \eqref{condition:over} 
states that the true label is assigned the highest probability under the fitted classifier, {which implicitly requires the classifier to be well estimated}.
{Due to the negative orientation of the score,  it is satisfied for example
for $\varepsilon=0$    if 
 $S(x, c) \geq S(x, i)$ when $i$ is the true label for $x$ and $c$ is any other label.}
\begin{proposition} \label{lemma:classification_overcoverage} Suppose that 
\eqref{condition:over}
    holds. 
Then 
$\P_1(Y_{n+1} \in \widehat{C}_n(X_{n+1{}})) \, \geq \, 1-\alpha. $
\end{proposition}

\begin{proof}
 {As, for all $q$, 
$\epsilon (F_1(q) - F_2(q)) 
= F_1 (q) - \tilde{F}(q)$, in view of Lemma \ref{lemma:lb} it suffices to show that 
$F_1(q) - \tilde{F}(q)  \, \geq\,  0$ for all $q$.} 
%{where  $F_2=\tilde{F}$ is the score cdf under the noisy labels.}
{Now, 
\begin{align}\label{eq:f1}
  F_1(q) = \sum_i P_i \P (S(X,i) \le q | Y = i) = \sum_{i=1}^K P_i F_1 (q, i ,i)  
\end{align}
and similarly
$\tilde{F}(q) = \sum_i \tilde{P}_i \P (S(X,i) \le q | \tilde{Y} = i) = \sum_{i=1}^K \tilde{P}_i \tilde{F} (q, i, i).$
Moreover as we assume ${Y}^c \perp X|Y$, 
given $Y=j$, knowing ${Y}^c$ does not contain any additional information on $S(X,i)$ and hence we have 
$\P (S(X,i) \le q \, | \, \tilde{Y} = i, Y =j )= \P (S(X,i) \le q | Y =j ),$
giving  
\begin{equation} \label{eq:f_tilde_mixture}
    \tilde{F}(q, i, i) = \sum_{j=1}^K \P( Y=j | \tilde{Y} = i) \P (S(X,i) \le q \,| \, \tilde{Y} = i, Y =j ) = \sum_{j=1}^K P_{ji} F_1 (q, i,j).
\end{equation} 
Using \eqref{eq:P_ji}, we obtain  
 $  \tilde{P}_i \tilde{F}(q, i, i) =    
   \sum_{j=1}^K {P}_j \tilde{P}_{ij} F_1(q; i,j) .$
 } 
{{Therefore,}
 \[
 {\tilde{F}(q) = \sum_{i=1}^K \tilde{P}_i \tilde{F}(q;i,i) 
      = \sum_{j=1}^K \Big[ P_j \tilde{P}_{jj} F_1(q;j,j) + \sum_{i: i\neq j} P_j \tilde{P}_{ij} F_1(q; i,j) \Big].}
 \]}
 {Thus, {with \eqref{eq:f1},}}
 \begin{align*}
    {F_1(q) - \tilde{F}(q)} 
      {= \sum_{j=1}^K P_j \Big[ F_1(q;j,j) - \tilde{P}_{jj} F_1(q;j,j) - \sum_{i: i\neq j} \tilde{P}_{ij} F_1(q; i,j) \Big]} \\
      {= \sum_{j=1}^K P_j \Big[ (1-\tilde{P}_{jj}) F_1(q;j,j) - \sum_{i: i\neq j} \tilde{P}_{ij} F_1(q; i,j) \Big].}
     \end{align*}
 {Using that $\sum_{i: i\ne j}\tilde{P}_{ij}=1-\tilde{P}_{jj}$ and the condition \eqref{condition:over}, we obtain the desired inequality;}
$  {F_1(q) - \tilde{F}(q)}
     \geq 
  \sum_{i=1}^K P_i (1 - \tilde{P}_{ii})\big( F_1(q;i,i) - \max_{{c:} c \neq i} F_1(q; c,i{)} \big)  |\, \geq 0 \, . 
    $
  % The last expression  is non-negative 
%due to Assumption \eqref{condition:over}. 
\end{proof}

%\paragraph{An Example: Uniform Noise}
\begin{example}[Uniform Noise] \label{ex:uniform}
As a classification example, assume that the corrupting noise chooses one of the $K$ labels uniformly at random, regardless of the true label, {so that ${{Y}}^c$ follows the uniform distribution on $[K]$, and 
%Moreover we 
assume
that $P_j = \frac{1}{K}$ also, so that the only difference between $Y$ and $Y^c$ is that $Y$ contains signal on $X$ whereas $Y^c$ does not. Then
$ \P (\tilde
{Y}=i ) = \frac{1}{K}$
and {by} Bayes{' } rule \eqref{eq:P_ji},
\begin{equation}
%\tilde
{P}_{ji} = \P ( Y = j{ \mid } \tilde{Y}=i )
= \frac{\P ( Y = j, \tilde{Y}=i )}{ \P (\tilde{Y}=i ) }
= 
\frac{\varepsilon}{K} + (1-\varepsilon) {\bf 1} (i = j).
\label{eq:classmodel} \end{equation}
%This matrix
The matrix %$P$ }is %symmetric {(with entries $P_{ji}=(1-\varepsilon)\mathbf{1}(i=j)+\varepsilon/K$)}; 
$%\tilde
{P} = (1-\varepsilon) I + \frac{\varepsilon}{K}\mathbf{1}\mathbf{1}^T$ {is symmetric,} 
where $\mathbf{1} \in \R^K$ is the vector of all ones, and $I \in \R^{K \times K}$ is the identity matrix. For $\varepsilon\neq 1$ (hence $1-\varepsilon>0$), $P$ is invertible. The Sherman{–}Morrison formula gives
\begin{align}
    %\tilde
    {P}^{-1} &= \frac{1}{1-\varepsilon}I - \frac{\frac{\varepsilon}{K(1-\varepsilon)^2}\mathbf{1}\mathbf{1}^T}{1 + \frac{K\varepsilon}{K(1-\varepsilon)}} 
    = \frac{1}{1-\varepsilon}I - \frac{\varepsilon}{K(1-\varepsilon)}\mathbf{1}\mathbf{1}^T. 
\end{align} 
}
\end{example}

\section{{Contamination Robust Conformal Prediction}}
\label{sec:remedy}
{It is a natural question whether in the case of over-coverage, the amount of over-coverage can be assessed.}
In the following, we abbreviate $$g(q) = F_1(q) - \tilde{F}(q).$$ With $S_{(i)}$ the $i^{th}$ order statistic of the scores in the calibration data with $i = \lceil (1-\alpha)(n+1) \rceil$, as for \eqref{eq:hatq}, we can {rewrite} the lower bound of {inequality} \eqref{eq:mixture_lb} as
$$ \P_1(Y_{n+1} \in \widehat{C}_n(X_{n+1})) \geq 1 - \alpha + \mathbb{E}[g(S_{(i)})], $$
where the expectation is taken over the calibration data. In general $\mathbb{E}[g(S_{(i)})]$ is not available; in the model \eqref{model}, by construction $|g(x)| \le \varepsilon,$ and under {first-order} stochastic dominance, if $\Pi_2$ {first-order} stochastically dominates  $\Pi_1$, we have $0 \le g(x) \le \varepsilon.$ 

For classification under label noise, when the assumptions for Proposition \ref{lemma:classification_overcoverage} are satisfied, we {next} show that under some extra conditions it is possible to {estimate  $g(S_{(i)})$, and use this estimate}  to adjust the nominal level to construct tighter prediction sets. 
The key insight is that, if we assume the corruption applied to the targets is independent of $X|Y$, we can write $F_1(q)$ in terms of the {mixture cdf} $\tilde{F}$, which allows us to use the contaminated samples to estimate $g(S_{(i)})$.

First, viewing \eqref{eq:f_tilde_mixture} as a matrix equation, we obtain an expression for $F_1(q; i, i)$ in terms of the inverse of $P$. 
Define the matrices $F_1(q), \tilde{F}(q) \in \R^{K \times K}$, where $(F_1(q))_{i,j} = F_1(q;  i,j)$ and $(\tilde{F}(q))_{i, j} = \tilde{F}(q;i,j)$. 
Similarly to Equation \eqref{eq:f_tilde_mixture} we have
\begin{equation} \label{eq:f_tilde_mixture2}
    \tilde{F}(q, i, j) = \sum_{k=1}^K \P( Y=k {\mid} \tilde{Y} = j)\, \P (S(X,i) \le q \, {\mid} \, \tilde{Y} = j, Y =k ) = \sum_{k=1}^K P_{kj} F_1 (q, i,k).
\end{equation}
so that $\tilde{F}(q) = F_1(q) P.$  
If $P$ is invertible, we can write
$ F_1(q) =  \tilde{F}(q) P^{-1}$ 
and read off that
\begin{equation} \label{eq:inverse_p_expression}
    F_1(q; i, i) = \sum_{j} P^{-1}_{ji} \tilde{F}(q; i, j).
\end{equation}
Using Equation \eqref{eq:inverse_p_expression}, we can now write $g(q)$ as
\begin{align}
g(q) 
&= \sum_{i=1}^K \big( P_i F_1(q;i,i) - \tilde{P}_i \tilde{F}(q;i,i)\big) = \sum_{i=1}^K \sum_{j=1}^K  P_i P^{-1}_{ji} \tilde{F}(q;i,j)  \; {-} \; \sum_{i=1}^K \tilde{P}_i \tilde{F}(q;i,i).  
\label{eq:over_as_f_tilde}
\end{align}
If $P_i$, $\tilde{P}_i$ and $P^{-1}$ are known, then \eqref{eq:over_as_f_tilde} suggests that it may be possible to estimate $g(q)$ using a plug-in estimate of $\tilde{F}$. In particular, we define 
\begin{equation}
    \tilde{F}_n(q, i, j) = \frac{\sum_{\ell=1}^n {\bf{1}} \big(S(X_\ell, i) \le q\big)\, {\bf{1}} \big({\tilde{Y}}_\ell =j\big)}{\sum_{\ell=1}^n {\bf{1}} \big({\tilde{Y}}_\ell =j\big)} 
    \label{eq:ecdf}
\end{equation}
as the empirical conditional cdf computed using the observed but contaminated data points drawn from $\tilde{F}$; {here $0/0 :=0$}. We construct an estimator of $g(q)$ as 
\begin{align} \label{eq:hatg}
    \hat{g}_n(q) 
    &= \sum_{i=1}^K \sum_{j=1}^K  P_i P^{-1}_{ji} \tilde{F}_n(q, i, j)  \; {-} \; \sum_{i=1}^K \tilde{P}_i \tilde{F}_n(q, i, i).
\end{align}

Then \eqref{eq:hatg} is a consistent estimator of $g(q)$ as $n$, the number of calibration data points, tends to infinity; see for example \cite{stute86}. This suggests the following approach: rather than choosing $i = \lceil (1-\alpha)(n+1) \rceil$, we take $i$ to be the smallest $i \in [n]$ such that
$$ \frac{i}{n+1} \geq 1 - \alpha - \hat{g}_n(S_{(i)}) $$ 
if such a choice of $i$ exists. Repeating the steps of the proof of Lemma \ref{lemma:lb}, we have that the coverage of this approach is lower bounded as
\begin{align*}
    \P_1(Y_{n+1} \in \widehat{C}_n(X_{n+1})) \geq 1 - \alpha + \mathbb{E}[g(S_{(i)}) - \hat{g}_n(S_{(i)})].
\end{align*}
As $n \rightarrow \infty$, we have that $\mathbb{E}[g(S_{(i)}) - \hat{g}_n(S_{(i)})] \rightarrow 0$ and we recover the desired coverage guarantee.

The calibration set is finite{; however}, and an appealing property of conformal prediction is that the coverage guarantee holds in finite samples. {Hence we derive a procedure which does not rely on taking a limit.}
To this end, we seek an upper bound $C{=C(n,\varepsilon)}$ such that 
$$ \mathbb{E}[\hat{g}{_n}(S_{(i)}) - g(S_{(i)})] \leq C. $$
Using this upper bound, 
we {now} take $i$ to be the smallest $i \in [n]$ such that
\begin{align}
    \label{ineq:crcp} 
\frac{i}{n+1} \geq 1 - \alpha - \hat{g}{_n}(S_{(i)}) + C.
\end{align} 
By the definition of $C$,  this choice of $i$ (if it exists) provides a coverage guarantee of 
\begin{align*}
    \P_1(Y_{n+1} \in \widehat{C}_n(X_{n+1})) &\geq 1 - \alpha + \mathbb{E}[g(S_{(i)}) - \hat{g}{_n}(S_{(i)})] - C 
    \geq 1 - \alpha. 
\end{align*}
    We call this method {\it Contamination Robust Conformal Prediction}, or CRCP for short. 
Next, using the notation from Section \ref{sec:classification_overcoverage},  we give  such a  theoretical upper bound ${C = C(n, \varepsilon)}$. 

\begin{theorem} \label{lemma:classification_estimator_bound}
{Set  $w^{(1)}_i = P_{i,i}^{-1}P_i - \tilde{P}_i$,  $w^{(2)}_{ij} = P_i P^{-1}_{ji}$, and %abbreviate 
$ b(n, j) = (1 - \tilde{P}_j)^n + \sqrt{\tfrac{\pi}{n \tilde{P}_j}} $.}
%$C_1 = \sum_{i=1}^K \left| P_{i,i}^{-1}P_i - \tilde{P}_i \right| $ and   $C_2 = \sum_{i=1}^K \sum_{j\neq i}^K P_i P^{-1}_{ji}$. 
Then
\begin{align} \label{thmbound}
\mathbb{E}\big[ \,| \hat{g}{_n}(S_{(i)}) - g(S_{(i)})| \,\big] 
\;\leq\; {C(n, \varepsilon)} := 
    \sum_{i=1}^K \Big( |w_i^{(1)}|\,b(n, i) + \sum_{{j: j\neq i}} |w_{ij}^{(2)}|\, b(n, j) \Big).
\end{align}
\end{theorem}

{In the bound \gr{$C(n, \varepsilon)$ in} \eqref{thmbound},  $w_i^{(1)}$ and $w_{ij}^{(2)}$ may depend on $\varepsilon$, but they do not depend on $n$, and  $b(n,j) \rightarrow 0$ as $n\rightarrow \infty$ as long as $K$ is fixed. Hence the bound will tend to 0 with increasing $n$, so that for any fixed $\alpha$, for large enough $n$ it will be possible to find $i=i(n)  < n+1 $  such that the inequality \eqref{ineq:crcp} holds.
Moreover, if $\varepsilon=0$ then $P$ is the $K \times K$ identity matrix and $\tilde{P}_i = P_i$, so that both $w_i^{(1)}$ and $w_{ij}^{(2)}, i \ne j$, equal zero; in this sense the bound \eqref{thmbound} is sharp.}

\gr{In this paper the noise level $\varepsilon$ is known. This assumption is not very realistic, but in some real data applications, it may be possible to give an upper bound on the level of contamination. In the next example, which we give before proving Theorem \ref{lemma:classification_estimator_bound},  $C(n,\varepsilon)$ increases monotonically in $\varepsilon$ and hence can be used as worst-case bound when a bound on $\varepsilon$ is available.}

\begin{example} \label{ex:unif2} 
In the uniform noise model from Example \ref{ex:uniform}, with 
\eqref{eq:classmodel}, we have that for the first coefficients, 
$\big| w_i^{(1)} \big| =  \Big| \frac{1}{1 - \varepsilon} \Big( 1 - \frac{\varepsilon}{K} \Big) \frac{1}{K} - \frac{1}{K} \Big| = \frac{\varepsilon(K-1)}{K^2(1 - \varepsilon) }$ and for $i \ne j,$
$\big| w_{ij}^{(2)} \big| = \frac{\varepsilon}{K^2(1 - \varepsilon)} $. Moreover,
$\tilde{P}_j
=\frac{1}{K}$
so that 
%$b(n,j) = \left( 1 -  \frac{1}{K} \right)^n  + \sqrt{\frac{\pi K}{n }} $. Thus,
the bound in \eqref{thmbound} {simplifies to $\frac{\varepsilon}{1- \varepsilon} \big( ( 1 - 1/K)^n + \sqrt{\frac{\pi K}{n}}\big) $. This expression}
tends to 0  monotonically when $n \rightarrow \infty $ and decreases monotonically to 0 when $\varepsilon \rightarrow 0$. {The increase in $K$ is $O (\sqrt{K})$. 

 \cite{penso2025conformal} give an alternative bound;  with probability $1 - \delta,$  we have 
$|\P_1 (Y_{n+1} \in \tilde{C}_n (X_{n+1}){)} - (1- \alpha)| \le \sqrt{\log (4 /\delta) / (2 n h^2)}$ where  $h = (1-\varepsilon)/(1 + \varepsilon).$ This bound  does not depend on $K$.
%, and indeed for $K$ large, the corresponding adaptive method {\it Noise-Aware Conformal Prediction (NACP)} in \cite{penso2025conformal} shows better performance in simulations than CRCP. 
However the bound in \cite{penso2025conformal} does not tend to 0 with $\varepsilon\rightarrow 0$. %Moreover,  in examples with a large number of classes, depending on the data one may perhaps use a continuous approximation and treat the classification problem as an estimation problem instead. Unfortunately 
Moreover, the method in \cite{penso2025conformal} does not easily generalise  to  the more general setting of Theorem \ref{lemma:classification_estimator_bound}.}
\end{example}

  We prove Theorem \ref{lemma:classification_estimator_bound}  in a number of steps. Recalling \eqref{eq:ecdf}, 
  %that
%$$\tilde{F}_n(q, i, j) = \frac{\sum_{\ell=1}^n {\bf{1}} (S(X_\ell, i) \le q) {\bf{1}} (y_\ell =j)}{\sum_{\ell=1}^n {\bf{1}} (y_\ell =j)}. $$
%where we assume the convention $0/0 = 0.$
 with 
$ n_j = \sum_{\ell=1}^n {\bf{1}} (y_\ell =j)$ we have 
$\tilde{F}_n(q, i, j) = \frac{1}{n_j}  \sum_{\ell=1}^n {\bf{1}} (S(X_\ell, i) \le q) {\bf{1}} (y_\ell =j).$ We  compare this to 
$\P ( S(X,i) \le q| \tilde{Y} = j) = \frac{\P ( S(X,i) \le q; \tilde{Y} = j) }{ \P  ( \tilde{Y} = j)}
.$ For ${\bf y}= (y_1, \ldots, y_n)$ let $A_j = A_j ({\bf y}) = \{ \ell: y_\ell = j\}$.
 Then as long as $A_j \ne \emptyset$, $$\tilde{F}_n(q, i, j) = \frac{1}{| A_j| }  \sum_{\ell \in A_j } {\bf{1}} (S(X_\ell, i) \le q) .$$

\begin{proof}
   Let $w^{(1)}_i = P_{i,i}^{-1}P_i - \tilde{P}_i$ and $w^{(2)}_{i,j} = P_i P^{-1}_{ji}$.  
   Using %the representation in Equation 
   \eqref{eq:over_as_f_tilde}, we can write
   \begin{align}
       \hat{g}(q) - g(q) 
       &= \sum_{i=1}^K \sum_{j=1}^K P_i P^{-1}_{ji} d(q, i, j) - \sum_{i=1}^K \tilde{P}_i d(q, i, i) \nonumber \\
       &= \sum_{i=1}^K \big(P_{i,i}^{-1} P_i - \tilde{P}_i\big) d(q,i,i) 
       + \sum_{i=1}^K \sum_{j \ne i} P_i P^{-1}_{ji} d(q,i,j) \nonumber \\
       &= \sum_{i=1}^K w^{(1)}_i \, d(q,i,i) + \sum_{i=1}^K \sum_{j\ne i} w^{(2)}_{ij}\, d(q,i,j).
   \end{align}
   Taking absolute values gives
   \[
   | \hat g(q)-g(q) | \le \sum_{i=1}^K |w^{(1)}_i|\,|d(q,i,i)| + \sum_{i=1}^K\sum_{j\ne i} |w^{(2)}_{ij}|\,|d(q,i,j)|.
   \]
   
To bound 
%\begin{equation}
$d_n(q, i, j) =  \tilde{F}_n(q, i, j) - \tilde{F}(q, i, j)$ %\label{dqij}
%\end{equation}
we use a technical lemma
which employs the Dvoretzky–Kiefer–Wolfowitz (DKW) inequality, in 
the form derived in \cite{massart_dkw}, to control the approximation of the cdf $\tilde{F}$ by the empirical cdf $\tilde{F}_n$. {The proof is in the Appendix}.
%, namely
\begin{lemma} \label{lemma:dqij_bound}
    For any $i, j \in [K]$, {with $b(n,j)$ as in Theorem \ref{lemma:classification_estimator_bound}} we have 
    % \begin{equation}
    %     \E \left [ \sup_q  d(q, i, j)   \right ] \leq \sqrt{\frac{\pi}{n\tilde{P}_j}} %- ( 1 - \tilde{P}_j)^n 
    %     \label{eq:bound1}
    % \end{equation}
    % and 
$ 
 \E \big[ \sup_q |d_n(q, i, j)| \big]  \leq   b(n,j). 
  %\sqrt{\frac{\pi}{n\tilde{P}_j}}
 %+ ( 1 - \tilde{P}_j)^n . 
 $ 
\end{lemma}

   Taking expectations of the supremum over all $q$ and using Lemma \ref{lemma:dqij_bound} proves the assertion.
\end{proof}

 \section{Experiments} \label{sec:experiments}

{In our experiments we compare CRCP to standard conformal prediction (CP) and to the APS variant of NACP by \cite{penso2025conformal}; for the latter purpose, our settings have uniform noise. In experiments by \cite{penso2025conformal} %give additional experiments showing 
{in settings with uniform noise},   ACNL  by \cite{sesia2023adaptive} never performs best, compared to NACP and CRCP. 
Also, the theoretical bound for ACNL is not quite as small as the one underpinning CRCP. 
Hence we do not include ACNL in the comparison.} 

%The following two classification examples with uniform class marginals are used in the simulation study  {in Section \ref{sec:experiments}}. For each of these two examples we later add contamination according to the uniform noise example. 

%\begin{example}
 %  Consider a Gaussian logistic regression model which we refer to as \texttt{Logistic}; we sample the features from a $p$-dimensional multivariate normal $X_i \sim \mathcal{N}(0, I)$ and then draw the label $Y_i$ from the distribution
%$ \P(Y_i = k|X_i=x_i) = \frac{e^{-x_i^T w_k}}{\sum_{j=1}^K e^{-x_i^T w_j}},  \; \; k \in [K]$
%where $w_k \in \R^p$ are also independent $\mathcal{N}(0, I)$ random vectors.  Then by symmetry the marginal class probabilities are uniform. 
%\end{example}

%\begin{example}
 %   In this example, which we call \texttt{Hypercube}, we use the $ \texttt{make\_classification}$ function implemented in the scikit-learn python library (\cite{scikit-learn}); clusters of points are generated about the vertices of a 5 dimensional hypercube with side lengths $2$ (one for each of the $K=5$ classes), where each cluster is distributed as a standard $5-$dimensional Gaussian centred at each vertex of the hypercube. The task is to predict which cluster each point belongs to based on its coordinates. Each feature vector is made up of $5$ informative features for each class, namely its coordinates, and we add 5 noise features to each vector $X_i$, for a total of $p=10$ features. Again, by symmetry, the marginal class probabilities are uniform. 
%\end{example}

\subsection{Synthetic Data} \label{subsec:synthadata} Here we apply 
CRCP to classification on two synthetic datasets with label noise. The first is a Gaussian logistic regression model which we refer to as \texttt{Logistic}; we sample the features from a $p$-dimensional multivariate normal $X_i \sim \mathcal{N}(0, I)$ and then draw the label $Y_i$ from the distribution
$ \P(Y_i = k|X_i=x_i) = \frac{e^{-x_i^T w_k}}{\sum_{j=1}^K e^{-x_i^T w_j}},  \; \; k \in [K]$
where $w_k \in \R^p$ are also independent $\mathcal{N}(0, I)$ random vectors;  
we set $p=10$ and $K=5$. By symmetry, the marginal label  probabilities are uniform.

For the second dataset, which we refer to as \texttt{Hypercube}, we use the $ \texttt{make\_classification}$ function implemented in the scikit-learn python library (\cite{scikit-learn}); clusters of points are generated about the vertices of a 5 dimensional hypercube with side lengths $2$ (one for each of the $K=5$ classes), where each cluster is distributed as a standard $5-$dimensional Gaussian centred at each vertex of the hypercube. The task is to predict which cluster each point belongs to based on its coordinates. Each feature vector is made up of $5$ informative features for each class, namely its coordinates, and we add 5 noise features to each vector $X_i$, for a total of $p=10$ features.  Again, by symmetry, the marginal label probabilities are uniform. 

For each dataset %we perform the following procedure; 
we sample 10,000 datapoints from the model for each of training, calibration and testing, and apply the uniform label noise model like the one in Example 
\ref{ex:uniform}.
%In both of these models
The marginal label probabilities are  $P_i = \frac{1}{K}$ for all $i \in [K]$, and 
$P_{ji} = \P (Y=j| \tilde{Y} =i) 
= \frac{\varepsilon}{K} + {\bf 1}(i=j)  (1- \varepsilon) $
as in \eqref{eq:classmodel}. We then use Bayes' rule to find $\hat{P}_i$.
Here we take the noise parameter $\varepsilon = 0.2$ and apply this contamination model to the training and calibration data. 
We fit a classifier on the training data, then use the fitted classifier to calibrate the conformal quantile for $\alpha=0.1$ using the Adaptive Prediction Set (APS) (\cite{aps}) score function with
{CP, CRCP, and NACP}. %both standard conformal prediction (which we refer to as CP),  the Contamination Robust Conformal Prediction (CRCP) method introduced in Section \ref{sec:remedy}, {and APS variant of the NACP method by \cite{penso2025conformal}}.
Finally, we construct prediction sets for the test data (which does not contain corrupted labels), and record the empirical coverage and average prediction set size for each of the two methods. {Here, as shown in Example \ref{ex:unif2}, the bound $C(n,\varepsilon)$ in \eqref{thmbound} decreases monotonically when $\varepsilon$ decreases, and hence can be seen as a worst-case bound if it is known that the noise level in the data does not exceed $\varepsilon$.} More details of the experimental setup are found in Appendix \ref{app:expdetails}.

We perform each experiment for four  classification models:  logistic regression (\texttt{LR}), gradient boosted trees (\texttt{GBT}), random forest (\texttt{RF}) and a multi-layer neural network (\texttt{MLP}). We repeat   each experiment 25 times;  Table \ref{tab:by_model} shows the mean and standard deviation of these quantities across repetitions. We see that CRCP 
consistently produced prediction sets with coverage close to the desired level of $90\%$, whereas standard conformal prediction (CP) {and NACP} grossly over-covered the clean labels.
Moreover CRCP
gave prediction intervals which are narrower, and hence more precise, than the ones obtained via CP {and NACP}.

\begin{table}[t]
\scriptsize
    \centering
\begin{tabular}{llllllll}
\toprule
  & \multicolumn{2}{r}{CP} & \multicolumn{2}{r}{CRCP} & \multicolumn{2}{r}{NACP} \\
Dataset & Model  & Coverage & Size & Coverage & Size & Coverage & Size \\
\midrule
\multirow[t]{4}{*}{Logistic} & GBT &  0.972 ±  0.006 &  3.051 ±  0.104 &  {\bf 0.919 ±  0.005} &  {\bf 2.246 ±  0.182} &  0.987 ±  0.004 &  3.622 ±  0.124 \\
 & LR &  0.981 ±  0.005 &  2.882 ±  0.099 &  {\bf 0.919 ±  0.005} &  {\bf 2.001 ±  0.153} &  0.981 ±  0.006 &  2.909 ±  0.158 \\
 & MLP &  0.977 ±  0.006 &  2.982 ±  0.101 &  {\bf 0.921 ±  0.003} &  {\bf 2.072 ±  0.191} &  0.985 ±  0.005 &  3.342 ±  0.159 \\
 & RF &  0.968 ±  0.006 &  3.158 ±  0.099 &  {\bf 0.920 ±  0.005}  &  {\bf 2.345 ±  0.181} &  0.940 ±  0.009 &  2.595 ±  0.218 \\
\cline{1-8}
\multirow[t]{4}{*}{Hyper-} & GBT &  0.982 ±  0.003 &  2.838 ±  0.062 &  {\bf 0.915 ±  0.006} &  {\bf 1.731 ±  0.083} &  0.990 ±  0.003 &  3.297 ±  0.132 \\
 {cube} & LR &  0.951 ±  0.008 &  3.496 ±  0.149 &  {\bf 0.917 ±  0.004} &  {\bf 3.050 ±  0.245} &  0.955 ±  0.006 &  3.567 ±  0.260 \\
 & MLP &  0.989 ±  0.002 &  2.711 ±  0.055 &  {\bf 0.915 ±  0.004} &  {\bf 1.492 ±  0.064} &  0.993 ±  0.002 &  3.121 ±  0.157 \\
 & RF &  0.983 ±  0.003 &  2.836 ±  0.072 &  {\bf 0.915 ±  0.006} &  {\bf 1.677 ±  0.092} &  0.965 ±  0.007 &  2.287 ±  0.132 \\
\cline{1-8}
% \bottomrule
\end{tabular}
        \caption{Coverage and {size} of prediction intervals, $\pm$1 standard deviation, in the classification experiment aiming for 90\% coverage, for {CP}, %standard conformal prediction (CP), %for 
        CRCP, and {for NACP with the APS version}. Bold: closest to 90\% coverage, and smallest in size. 
        }
     \label{tab:by_model}
     \vspace{-0.75cm}
\end{table}

%\paragraph{Ablation Study}
To %better understand 
{assess} the dependence of CRCP 
on the contamination strength we perform on ablation study on the noise parameter $\varepsilon$. {Although the standard %Huber contamination
model \eqref{model} assumes %that 
$\epsilon \le 1/2$,  our theoretical guarantees do not require this assumption; here we also include larger values of $\epsilon$.} 
Using the \texttt{Logistic} dataset with the logistic regression \texttt{LR} model, we re-run the experiment but vary the parameter $\varepsilon$, {Table \ref{table:epsresults} shows the results. For zero noise, CRCP and CP co-incide (up to random variation) and give the desired coverage, while NACP severely over-covers the clean labels. %has far too much coverage. The
CRCP has closest to desired coverage and smallest prediction interval size in the range $[0, 0.4]$. Only when the noise ratio is at least 50\% does NACP perform better than CP and CRCP.}

\begin{table}[h]
    \centering
    \scriptsize
\begin{tabular}{ccccccc}
\toprule
\multirow{2.5}{*}{$\epsilon$} & \multicolumn{2}{c}{CP} & \multicolumn{2}{c}{CRCP} & \multicolumn{2}{c}{NACP} \\ \cmidrule(lr){2-3} \cmidrule(lr){4-5} \cmidrule(lr){6-7} & Coverage & Size & Coverage & Size & Coverage & Size \\
\midrule
0.00 & \textbf{0.901±0.005} & \textbf{1.787±0.165} & \textbf{0.901±0.006} & \textbf{1.786±0.167} & 0.992±0.005 & 3.135±0.121 \\
0.01 & 0.907±0.005 & 1.829±0.167 & \textbf{0.902±0.006} & \textbf{1.801±0.165} & 0.993±0.005 & 3.261±0.161 \\
0.03 & {0.917±0.005} & 1.911±0.167 & \textbf{0.903±0.006} & \textbf{1.823±0.161} & 0.994±0.005 & 3.362±0.180 \\
0.05 & 0.927±0.006 & 1.998±0.160 & \textbf{0.905±0.006} & \textbf{1.845±0.156} & 0.993±0.006 & 3.352±0.200 \\
0.10 & 0.950±0.005 & 2.256±0.160 & \textbf{0.910±0.006} & \textbf{1.904±0.157} & 0.991±0.006 & 3.232±0.142 \\
0.15 & 0.969±0.005 & 2.559±0.138 & \textbf{0.914±0.005} & \textbf{1.952±0.156} & 0.987±0.006 & 3.073±0.138 \\
0.20 & 0.981±0.005 & 2.882±0.099 & \textbf{0.919±0.005} & \textbf{2.001±0.153} & 0.981±0.006 & 2.909±0.158 \\
0.25 & 0.988±0.004 & 3.187±0.077 & \textbf{0.926±0.005} & \textbf{2.066±0.169} & 0.976±0.006 & 2.785±0.160 \\
0.30 & 0.991±0.004 & 3.451±0.050 & \textbf{0.934±0.007} & \textbf{2.156±0.181} & 0.972±0.005 & 2.707±0.182 \\
0.35 & 0.993±0.003 & 3.653±0.044 & \textbf{0.941±0.007} & \textbf{2.250±0.197} & 0.966±0.006 & 2.619±0.212 \\
0.40 & 0.994±0.002 & 3.807±0.039 & \textbf{0.950±0.008} & \textbf{2.376±0.223} & 0.961±0.007 & 2.552±0.226 \\
0.50 & 0.995±0.002 & 4.023±0.037 & 0.969±0.009 & 2.806±0.302 & \textbf{0.953±0.009} & \textbf{2.488±0.236} \\
0.60 & 0.996±0.002 & 4.182±0.024 & 0.995±0.005 & 4.133±0.498 & \textbf{0.949±0.011} & \textbf{2.489±0.261} \\
0.80 & 0.995±0.002 & 4.385±0.019 & 0.995±0.002 & 4.386±0.019 & \textbf{0.945±0.022} & \textbf{2.706±0.368} \\
\bottomrule
\end{tabular}
    \caption{The coverage and width for different $\epsilon$ for the logistic regression including NACP-APS. Bold text indicates the coverage closest to the desired 90\% and the smallest size. %\wx{WKcomment: running a bit more on different $\epsilon$ for us to see and show that NACP works better when $\epsilon$ not too big, i.e. up to $\epsilon=0.4$ CRCP is better; then NACP takes over. We may not need to keep the whole list due to page limit.} 
    }\label{table:epsresults}
    \label{tab:placeholder}
\vspace{-0.5cm}
\end{table}
%{In general, the empirical frequencies of $\tilde{P}_i$ are observed and the  estimation of the matrix $P$ can be carried out for example as in \cite{zhu2022beyond}. Estimating the distribution for the true labels may be not as straightforward; perhaps this distribution is known from previous studies,  or one may be able to assume that the true and the corrupted labels have the same marginal distribution, with the difference between the two distributions being that the true labels are related to the features of the data, whereas the corrupted labels are not related to the features.} 

%, and plot the mean and standard deviation of the empirical coverage and size of the prediction sets in Figure \ref{fig:logistic_by_eps}.  
%{While coverage and interval size increases with increased contamination, the deviation in coverage and the increase in size are much lower in CRCP compared to CP.} {We also see that CRCP always maintains at least $90\%$ coverage, which supports our bound in Theorem \ref{lemma:classification_estimator_bound}.} 

%\begin{figure}[t!]
 %   \centering     \includegraphics[width=\textwidth]{figs/logistic_by_eps_plot_byhand.pdf}
 %   \caption{Coverage and size for CP  and CRCP  in the logistic regression example with the LR model, and varying $\varepsilon$. The {red} line in the left panel indicates the target 90\% coverage and the shaded regions indicate one standard deviation. 
  %  }
    %\label{fig:logistic_by_eps}
 %        \vspace{-0.5cm}
%\end{figure}

\subsection{Real Data with Label Noise}  \label{subsec:realdata} 
 
Here we illustrate CRCP on the contaminated CIFAR-10 dataset known as CIFAR-10N{\footnote{Publically available to download at \url{http://noisylabels.com/}.} introduced in \cite{wei2022learning}. For this dataset, 3 independent workers were asked to assign labels to CIFAR-10 datasets collected from Amazon Mechanical Turk; for details see \cite{wei2022learning}. The CIFAR-10 dataset contains 60,000 colour images in 10 classes (such as airplane, automobile, bird), with 6,000 images per class. There are 50,000 training images and 10,000 test images. 
There are six sets of labels provided for each training image representing different noise patterns:
\begin{itemize}[noitemsep,topsep=3pt,leftmargin=*]
    \item {\tt Clean}: This is the CIFAR-10\gr{H} dataset \gr{from  \cite{peterson2019human} which is assumed to have} 
    a noise rate of 0 \%. 
    \item {\tt Aggr}: In this dataset the label is assigned by majority voting, and picked randomly from the three submitted labels when there is no majority; the noise rate is 9.03\%. 
    \item {\tt R1}: The assigned noisy label is the first submitted label for each image. 
These labels have  a noise rate of 17.23\%. 
 \item {\tt R2}: The   assigned noisy label is the second submitted label for each image. 
These labels have a noise rate of 18.12\%. 
 \item {\tt R3}: The   assigned noisy label is the third submitted label for each image. 
These labels have  a noise rate of 17.64\%. 
    \item {\tt Worst}: If there are any wrongly annotated labels then the worst label is randomly selected from the wrong labels. These labels have a noise rate of 40.21\%.
\end{itemize}
%, based on}
%by 
%empirical frequencies.
%(using the notation from Section \ref{sec:classification_overcoverage}).  
{%In this example it is
{We} assume that CIFAR-10H is the ground truth, so that the empirical distributions of the clean labels are available, and the matrix $P$ can be estimated.}  For each dataset, {in the notation from Section \ref{sec:classification_overcoverage},}  we \gr{use the estimates for} 
$P_i$  and $\tilde{P}_i$, for $i=1, \ldots, K$, as well as the matrix $P$ obtained in \cite{wei2022learning}.
We split the 50,000 images that have noisy, human-annotated labels into 20,000 images for training, 10,000 for validation, and 20,000 for calibration. For each of the six %different 
label noise settings, we fine-tune a pre-trained ResNet-18 (\cite{he2016deep}) model using the training and validation sets (see Appendix \ref{app:expdetails} for details).
For each repetition we subsample 10,000 calibration data from the remaining 20,000 images with human-annotated labels to calibrate the conformal prediction procedures, and subsample 5,000 test data from the provided set of 10,000 images which were not humanly annotated to evaluate the constructed prediction sets. %Similar to the synthetic experiments, 
We perform 25 repetitions of this experiment with different seeds {and at a desired coverage of 90\%; we} %and 
display the results in Table \ref{tab:cifar10n}. 
%At a desired coverage of 90\%, 
{F}or the clean data both CP and CRCP achieve within two standard deviations of the desired coverage. %For noisy data, though, 
{W}ith increasing noise ratio,  CP {and NACP} obtain considerable over-coverage, while CRCP stays within two standard deviations of the desired noise level even for  
dataset with the highest noise level. 
Moreover the prediction intervals obtained by CRCP are considerably narrower, and hence more precise, than those obtained by CP {and NACP}. 

\begin{table}[t]
\centering
\scriptsize
%\small 
\begin{tabular}{ccccccc}
\toprule
\multirow{2.5}{*}{Noise Type} & \multicolumn{2}{c}{CP} & \multicolumn{2}{c}{CRCP} & \multicolumn{2}{c}{NACP} \\ 
\cmidrule(lr){2-3} \cmidrule(lr){4-5} \cmidrule(lr){6-7}
{} &         Coverage &             Size &   Coverage &             Size &         Coverage &             Size \\
\midrule
\texttt{Clean}      &   \textbf{0.900 ±  0.005} &   \textbf{1.507 ±  0.019} &   0.909 ±  0.005 &   \textbf{1.507 ±  0.019} & 0.999 ± 0.004	& 7.790 ± 0.574\\
\texttt{Aggr}       &   0.940 ±  0.003 &   2.003 ±  0.027 &   \textbf{0.899 ±  0.005} &   \textbf{1.550 ±  0.019} & 0.998 ± 0.001 & 6.345 ± 0.152\\
\texttt{R1}         &   0.973 ±  0.002 &   2.997 ±  0.053 &   \textbf{0.902 ±  0.005} &   \textbf{1.672 ±  0.022} & 0.996 ± 0.001 & 5.463 ± 0.158 \\
\texttt{R2}         &   0.977 ±  0.002 &   3.177 ±  0.066 &   \textbf{0.903 ±  0.006} &   \textbf{1.658 ±  0.021} & 0.998 ± 0.001 &	8.730 ± 0.062\\
\texttt{R3}         &   0.973 ±  0.002 &   3.042 ±  0.079 &   \textbf{0.898 ±  0.006} &   \textbf{1.636 ±  0.027} & 0.992 ± 0.001	& 4.997 ± 0.151\\
\texttt{Worst}      &   0.990 ±  0.001 &   5.473 ±  0.078 &   \textbf{0.917 ±  0.009} &   \textbf{2.189 ±  0.093} & 0.985 ± 0.001 &	5.393 ± 0.131\\
\bottomrule
\end{tabular}
\caption{Coverage and size of prediction intervals, $\pm$1 standard deviation, %for prediction intervals 
aiming for 90\% coverage, %for standard conformal prediction (CP) and for Contamination Robust Conformal Prediction (CRCP) 
on the CIFAR-10N dataset. CRCP gives prediction intervals which are very close to the desired 90\% coverage across all noise levels, and does not over-cover much even in the noiseless case, while CP {and NACP} intervals show considerable over-coverage on all but the noiseless data. Moreover the CRCP intervals are narrower than the CP {and the NACP} intervals. %\wx{WKComment: NACP in this case shows more overcoverage, at $\alpha=0.1$, i.e. the coverage is near 0.99 instead of 0.90; correspondingly, the size of the set is much larger as well.}
}
\label{tab:cifar10n}
\vspace{-0.5cm}
\end{table}

%\vspace{-0.5cm} 

{Additional experiments in \cite{penso2025conformal} %in settings with uniform noise 
show that 
%The preprint 
CRCP performs better %than ACNL  by \cite{sesia2023adaptive} and
NACP %by \cite{penso2025conformal} 
when the number of classes is small, while NACP performs better when the number of classes is large.} 
%{ACNL never performs best in their experiments. 
%In these experiments, $1- \alpha = 0.9$ and $\epsilon = 0.1$ or $0.2$; in particular, $\epsilon \ge \alpha.$ }{\color{red}Wenkai: results? we manage to have an example on logistic regression} 

\section{Discussion}
\label{sec:discussion}

{This paper improves on the literature regarding conformal prediction under noisy labels by providing improved theoretical results and a new method, CRCP, for adjusting the conformal prediction algorithm under label noise. In contrast to \cite{barber2023conformal}, \cite{sesia2023adaptive}, and \cite{penso2025conformal}, the theoretical focus is on bounds on the correction factor which tend to 0 when the contamination tends to 0.}

{In related work,  instead of assuming a particular contamination model, \cite{bashari2025robust} address  the related classification task of outlier detection.}} \cite{einbinder2022conformal} study the robustness of split conformal prediction in a related setting where it is assumed that the entire calibration set is observed under \textit{label noise}. In the regression setting, this means noisy observations $S_i^{noisy}$ of the true scores $S_i$ are observed as 
$ S^{noisy}_i = S_i + Z $ 
where $Z$ is a random variable. In contrast, our models as detailed in Section \ref{sec:cp_mixture} consider the classical Huber contamination model in which some data are clean and 
only a fraction of the data are corrupted. As only continuously observed label noise is considered in \cite{einbinder2022conformal} they conclude that for regression, in all but pathological cases conformal prediction continues to provide coverage. They provide results similar to those introduced in Subsection \ref{sec:over_under} {of Section \ref{sec:cp_mixture}}, although they only consider the case where conformal prediction still provides (conservative) coverage and do not provide estimates for the coverage probabilities or efficiency. 

%\paragraph{Wider Related Works} 
{Wider related works} have addressed conformal prediction under different types of distribution shift including covariate shift (\cite{tibshirani2019conformal}) and various different online or adversarial settings (\cite{gibbs2021adaptive, zaffran2022adaptive, faci, bastanipractical}). Our outlier setting can be seen as a specific case of distribution shift where the quantile is calibrated over a different distribution to the test data, and hence methodology from this literature could be used to address data corruption if 
it is known which of the observations are corrupted (an assumption which we do not make in this paper).
%the presence of corrupted data is known. 
A related approach is that of \cite{cauchois2024robust}, who study the distribution shift setting; here they propose to first estimate the magnitude of distribution shift between calibration and test data before using this estimate to adjust a conformal prediction procedure. While we assume that data corruptions are sampled from some unknown but fixed distribution, \cite{gendler2022adversarially} consider the setting where the data is perturbed adversarially, and apply a randomized-smoothing approach to estimate an adjustment to recover coverage guarantees.  
{In \cite{angelopoulos2023conformal}, conformal prediction is extended to control the expectation of monotone loss functions, with one of the examples being transductive learning with a special type of distributional shift, namely that the conditional distribution of $Y$ given $X$ remains the same in training and test domain. In this case the distributional shift reduces to a covariate shift, which can be tackled using a weighted procedure.} 

{There are also related papers on group conditional coverage: if the group membership of each sample is known, this problem can be tackled by configuring different thresholds for the different groups (\cite{jung2022batch, gibbs2023conformal}). In our setting, there would be two groups: clean or contaminated. However, we assume that the group membership is not known.}

%{In \cite{sesia2023adaptive}, an alternative bound is derived; it has $\log(K)$-dependence in the number of classes, but it does not tend to 0 as $\varepsilon \rightarrow 0.$ Their method is abbreviated ACNL, for {\it Adaptive Conformal Classification with Noisy Labels}. In the experiments from \cite{penso2025conformal}, CRCP performs better than ACNL on all data sets.}

{Finally we highlight some  directions for future work.} 
This paper has illustrated that conformal prediction may be misleading when the data follow a Huber-type contamination model. In the classification setting we were able to offer a remedy under additional assumptions, via Theorem \ref{lemma:classification_estimator_bound}. The proof relies heavily on the discrete nature of the problem. %, which gives rise to a binomial distribution.
{For regression under continuous label noise, with additional assumptions  \cite{cohen2025efficient} use a  deconvolution argument to create an adaptive strategy for conformal prediction.}
%; it is not immediately obvious how to generalise it to continuous settings. Yet,   
{S}imilar remedies may be possible in other settings, such as that of functional regression. Exploring these settings will be part of future work.
Moreover, the assumption of independent observations is perhaps too strong. 
\cite{barber2023conformal} extended conformal prediction to the non-exchangeable setting using a data re-weighting approach. {It may  be possible to use a similar re-weighting to extend CRCP to non-exchangeable observations.}

% \section{Conclusion}
% The conclusion text goes here.

% \paragraph{Acknowledgements.}

\vskip6pt

{\bf Acknowledgements. }{The authors would like to thank Aleksandar Bojchevski, Tom Rainforth, Aaditya Ramdas  and Matteo Sesia for helpful discussions, and {Daniel Reinert for computing support}. G.R. acknowledges support from EPSRC grants EP/T018445/1, EP/W037211/1, EP/V056883/1, and EP/R018472/1.
W.X. acknowledges support from EPSRC grant EP/T018445/1;  he is also supported by the DFG (German Research Foundation) %under Germany’s Excellence Strategy
– EXC number 2064/1 – Project number 390727645.}

%%%%%%%%%% Insert bibliography here %%%%%%%%%%%%%%
%\bibliographystyle{abbrvnat}
%\printbibliography

\bibliographystyle{abbrv} 
\bibliography{main}
\clearpage
\appendix
% \section{}

\section*{Appendix}

\subsection{Further proofs}

\begin{proof}[{\bf Proof of Lemma \ref{lem:rep}}]
{We have
\begin{align*}
    \E_{\tilde \pi}\!\big[ w_q(\tilde T)\,h(\tilde T) \big]
  &= (1 - \Ftilde(q)) \int_{-\infty}^q (\pi_2(x) - \pi_1(x) ) \, dx - \Ftilde(q)) \int_{q}^\infty(\pi_2(x) - \pi_1(x) ) \, dx\\
  &= (1 - \Ftilde(q)) \int_{-\infty}^q (\pi_2(x) - \pi_1(x) ) \, dx +  \Ftilde(q)) \int_{-\infty}^q (\pi_2(x) - \pi_1(x) ) \, dx\\
  &= F_2(x) - F_1(x), 
\end{align*}
proving the first claim. 
Here we used  that $\int_{-\infty}^q \pi_i (s) \, dx = 1 - \int_q^\infty \pi_i (s) \, dx$.  
}
%The identity is immediate from the definition of $\tilde f_q$ in 
%\eqref{steinsol}, combined with the Stein operator applied under $\tilde \pi$. {It is actually immediate by direct verification; no Stein needed here, except perhaps as motivation.}
The second claim follows by a direct calculation of the variance of $w_q$. 
\end{proof}
To prove Lemma \ref{lemma:dqij_bound}, first we prove a  lemma.
\begin{lemma}\label{lem:binbound}For any $0\le p \le 1,$
$\sum_{k=1}^n {n \choose k} \frac{1}{\sqrt{k}} p^k (1-p)^{n-k} \le \sqrt{ {2}/({np})}. $
\end{lemma}

\begin{proof}
Let $B\sim Bin (n-1, p)$. We observe that 
\begin{align*}
\sum_{k=1}^n {n \choose k} \frac{1}{\sqrt{k}} p^k (1-p)^{n-k} 
    & = n p \sum_{\ell=0}^{n-1}{{n-1} \choose \ell} \frac{1}{{(\ell+1)^{3/2}}}  p^\ell (1-p)^{n-1-\ell}= n p \, \E [(1+B)^{-3/2}]. 
\end{align*}
%Next we assert that 
Next we prove that
$
\E [(1+B)^{-\frac32}] \leq \sqrt{2}(n p)^{-\frac{3}{2}} .$ 
By the Cauchy-Schwarz inequality,  
\begin{align*}
    \E [(1+B)^{-\frac32}] = \E \Big[ \frac{1}{B+1} \sqrt{\frac{1}{B+1}} \Big] \leq \sqrt{\E \Big[ \frac{1}{(B+1)^2} \Big] \E \left[ \frac{1}{B+1} \right].}
\end{align*} 
To bound $\E [(1+B)^{-2}]$, we use 
%first introduce 
the notion of \textit{size bias} distributions; a random variable $X^s$ has the size bias distribution of a nonnegative rv $X$ if and only if 
$ \E X f(X) = \E X \, \E f(X^s) $
for all measureable functions $f$.
One can show that if  $X_n \sim Bin (n, p), $ then $1 + X_{n-1}$ has the $X_n$-size bias distribution (see, for example \cite{arratia2019size}). So $1+B$ has the $X_n$-size bias distribution with $p=\tilde{P}_j$. Moreover, $1/( 1 + x)^2 \le 2/( 2 + x)^2$ and 
\begin{align}
    \E (1 + B)^{-2} \le 
 2 \, \E (2 + B)^{-2} &=  \frac{2}{np} \E \frac{X_n}{(1 + X_n)^2} %\nonumber %\\
%&
<    \frac{2}{np} \E \left\{ \frac{1}{(1 + X_n)} \right\}  %\nonumber \\
%&
=   
\frac{2}{np} \frac{1}{(n+1)p} \P (X_n \ge 1)  \nonumber
%\label{eq:1+B-2_bound}
\end{align} 
where we used size biasing twice, once with $f(x) = \frac{1}{1+x}$ and once with $f(x) = \frac{1}{x} 1 (x \ge 1).$ Using %
that 
   $\E \left[ \frac{1}{B+1} \right] \leq \frac{1}{np}$ %\end{eqnarray} 
 (see \cite{storey_fdr}, Lemma 3),
gives 
$ \E [(1+B)^{-\frac32}] \leq \sqrt{2}(n p)^{-\frac{3}{2}}.$ 
%yielding the assertion. 
\end{proof}

%Now we give a proof of Lemma \ref{lemma:dqij_bound}. 
\begin{proof}[Proof of Lemma \ref{lemma:dqij_bound}]
We have  \begin{eqnarray} 
 d_n(q, i, j) %&= & \tilde{F}_n(q, i, j) - \tilde{F}(q, i, j) \nonumber \\
 &=& \sum_{A} {\mathbf{1}} ( A_j = A) \Big\{ \frac{1}{|A|} \sum_{\ell \in A} {\bf{1}} (S(X_\ell, i) \le q )
  - \P (S(X, i) \le q | \tilde{Y}=j)
 \Big\} \nonumber \\
 &=& 
 \sum_{A \ne \emptyset} {\mathbf{1}}( A_j = A) \Big\{ \frac{1}{|A|} \sum_{\ell \in A} {\bf{1}} (S(X_\ell, i) \leq q )
  - \P (S(X, i) \leq q | \tilde{Y}=j)
 \Big\} \nonumber \\
 && \qquad - {\bf 1} (A_j = \emptyset) \P (S(X, i) \leq q | \tilde{Y}=j) \label{secondbound} 
 %\\
 %&\leq& \sum_{A \ne \emptyset} {\bf 1}( A_j = A) \left\{ \frac{1}{|A|} \sum_{\ell \in A} {\bf{1}} (S(X_\ell, i) \leq q )
 % - \P (S(X, i) \leq q | \tilde{Y}=j)
 %\right\}
 . 
\end{eqnarray} 
Taking expectations and supremum,
\begin{eqnarray} 
 \lefteqn{\E \sup_{q} | d_n(q, i, j) | }\label{eq_d_expanded}  \\ 
 &\leq& \sum_{A \ne \emptyset}\P( A_j = A)  \E \Big[ \sup_{q} \Big| \frac{1}{|A|} \sum_{\ell \in A} {\bf{1}} (S(X_\ell, i) \leq q )
  - \P (S(X, i) \leq q |
  \mid \tilde{Y}=j) \Big| \, \, \mid A_j = A
 \Big]  \nonumber \\
 && + \P (A_j = \emptyset) \nonumber.  
\end{eqnarray} 
Now if $A_j=A$ then $\ell \in A $ if and only if $\tilde{Y}_\ell =j$. Hence using the independence of the observations,
${\bf 1}(\ell \in A) \P (S(X_\ell, i) \leq q | A_j=A) 
=\P (S(X_\ell, i) \leq q \mid \tilde{Y}_\ell = j).$
 {In general, for   i.i.d.\, random variables $X_1, \ldots, X_n$ with cdf $F$ and   empirical cdf $F_n$, 
    \begin{align}
\label{lem:DKWbound}
  \E [ \sup_x | F_n(x) - F(x)|] \le \sqrt{{\pi}/({2 n})}. 
    \end{align}
    To see this, the refined DKW inequality by \cite{massart_dkw} yields that 
 \begin{align*} 
\E \left[  \sup_q  | F_n(q) - F(q) 
| \right] 
 &= \int_0^\infty \P \left(\sup_q | F_n(q) - F(q) | > \delta  \right) d \delta \nonumber  \leq 2 \int_0^\infty e^{-2|n|\delta^2} d\delta 
  =  \sqrt{\frac{\pi}{2|n|}} .
\end{align*}
 Next we apply \eqref{lem:DKWbound}} to obtain that for any fixed set $A \ne \emptyset$,
 \begin{align} 
    & \E \Big[  \sup_q  \Big|  \frac{1}{|A|} \sum_{\ell \in A} {\bf{1}} (S(X_\ell, i) \le q )
  - \P (S(X, i) \le q | \tilde{Y}=j)
 \Big|\, \, \mid A_j = A    \Big ] 
  \leq \sqrt{\frac{\pi}{2|A|}} .\label{eq:dkw_fixed_A_bound}
\end{align}
As $\tilde{Y}_i, i=1, \ldots, n$ are i.i.d., for a set $A$, 
$\P (A_j = A)$ depends only on the  size $|A|$ of $A$, and $| A_j| = |A_j (\tilde{Y}_i, i=1, \ldots, n)| 
= | \{ \ell: \tilde{Y}_\ell =j\} | \sim Bin (n, \tilde{P}_j)$. 
Using  %the inequality
\eqref{eq:dkw_fixed_A_bound} in  \eqref{eq_d_expanded} and unconditioning gives 
\begin{eqnarray*} 
\E  \sup_q  |  d_n(q, i, j) |  &\leq& 
%- ( 1 - \tilde{P}_j)^n + 
 \sum_{A \ne \emptyset} \sqrt{\frac{\pi}{2}} \P (A_j= A) \frac{1}{\sqrt{|A|}}  
+  \P (A_j = \emptyset) \nonumber \\ 
&= & \sqrt{\frac{\pi}{2}}\sum_{k=1}^n {n \choose k} \frac{1}{\sqrt{k}} \tilde{P}_j^{k} ( 1 - \tilde{P}_j)^{n-k} +  \P (A_j = \emptyset) \nonumber  \\
&= &\sqrt{\frac{\pi}{2}}
n\tilde{P}_j \sum_{\ell=0}^{n-1} {{n-1} \choose \ell} \frac{1}{{(\ell+1)^{3/2}}} \tilde{P}_j^{\ell} ( 1 - \tilde{P}_j)^{n-1-\ell} +  \P (A_j = \emptyset) \nonumber \\
& \leq &  
  \sqrt{\frac{\pi}{n\tilde{P}_j}}
 + ( 1 - \tilde{P}_j)^n 
%&=& %- ( 1 - \tilde{P}_j)^n+ 
%\sqrt{\frac{\pi}{2}} n\tilde{P}_j \E [(1+B)^{-\frac32}]  +  \P (A_j = \emptyset) 
\end{eqnarray*}
where {we used Lemma \ref{lem:binbound}
as well as 
  $ \P (A_j = \emptyset) = ( 1 - \tilde{P}_j)^n $
in the last step.}
%Combining \eqref{eq:e_bound} with %\gr{Equation \eqref{secondbound} }
%the bound 
%{\eqref{sizebiasargument}} %\eqref{eq:dkw_fixed_A_bound}
%we thus obtain 
%\begin{eqnarray*} 
%$
% \E \sup_q |d{_n}(q, i, j)|  \leq  %
 % \sqrt{\frac{\pi}{n\tilde{P}_j}}
 %+ ( 1 - \tilde{P}_j)^n $
 %\end{eqnarray*}
 %as claimed.
\end{proof}

\subsection{Experimental Details for Section \ref{sec:experiments}} \label{app:expdetails}
For both the synthetic and CIFAR-10N experiments we used the implementation of the APS scoring function provided by the authors of %the original paper 
\cite{aps}. The experiments were %all 
run on a single machine with an AMD Ryzen 7 3700X 8-Core Processor and an NVIDIA GeForce RTX 2060 SUPER GPU. The total running time to reproduce all the synthetic experiments in the paper is around 8 hours, and for the CIFAR-10N it is about 2 hours in total on this hardware.

For the synthetic data experiments, %all 
the classifiers were implemented using the scikit-learn \cite{scikit-learn} library and were trained using their default hyper-parameters. 
For the CIFAR-10N experiments, we used the ResNet-18 model available in the torchvision library %. model 
which is pretrained on the ImageNet database (%and
available at \url{https://pytorch.org/vision/main/models/generated/torchvision.models.resnet18.html}), and replaced the final layer to match the number of classes, which is 10 in our case. We initialise the weights to those of a ResNet-18 pre-trained on the Imagenet dataset (also available in the torchvision library) and trained each model for 30 epochs using the Adam \cite{adam} optimiser, using a batch size of $128$ and an initial learning rate of $1e^{-3}$, which is decayed by a factor of $10$ any time the training loss does not decrease for 3 epochs. We evaluate the validation loss at every epoch, and pick the model with the lowest validation loss overall. 

\end{document}